\documentclass{article}
\usepackage[T1]{fontenc}
\usepackage{iclr2027_conference,times}
\usepackage{amsmath,amssymb,amsthm,booktabs,graphicx,etoolbox,xcolor,hyperref,url}
\usepackage{tikz}
\usetikzlibrary{arrows.meta,fit,backgrounds}
\makeatletter
\newcounter{algorithm}
\renewcommand{\thealgorithm}{\arabic{algorithm}}
\newcommand{\fps@algorithm}{tbp}
\newcommand{\ftype@algorithm}{4}
\newcommand{\ext@algorithm}{loa}
\newcommand{\fnum@algorithm}{Algorithm~\thealgorithm}
\newenvironment{algorithm}{\@float{algorithm}}{\end@float}
\makeatother
\newtheorem{theorem}{Theorem}[section]
\newtheorem{proposition}[theorem]{Proposition}
\newtheorem{lemma}[theorem]{Lemma}

\newcommand{\R}{\mathbb{R}}

\graphicspath{{figures/}}
\title{Broken Symmetry in BF16 Attention: Why FlashAttention Gradients Blow Up Late in Training}
\author{Junlin Chen$^{1,5}$\quad Daize Dong$^{1}$\quad Huanwei Di$^{1}$\quad Haolong Jia$^{1}$\\
\textbf{Jiawei Wu$^{1}$\quad Haotian Xie$^{1}$\quad Mingkai Zheng$^{1}$\quad Yang Li$^{1}$}\\
\textbf{Leshang Chen$^{2}$\quad Huishu Wang$^{3}$\quad Eric P. Xing$^{4,5}$\quad Hongyi Wang$^{1,5}$}\\
$^{1}$Rutgers University\quad $^{2}$Oracle\quad $^{3}$New York University\\
$^{4}$MBZUAI\quad $^{5}$Carnegie Mellon University\\
\texttt{junlin.chen110@rutgers.edu}}

\iclrfinalcopy
\newcommand{\GProjAccuracyTable}{%
\begin{table}[t]
\centering\small
\caption{Median full-tensor relative $L_2$ error (\%) across eight captures from the 450M model, comparing public BF16 gradients with uncast FP64 references. GProj includes the $dK$ correction unless stated otherwise. $^\dagger$The same input files are used, but the FP32 study evaluates the same mathematical target with a separate FP64 implementation (Appendix~\ref{app:gproj-experiments}). $^\ddagger$QY is the measured source-faithful CUDA/C++/ATen implementation with BF16 O/LSE and the original tile clamp, evaluated with a separately implemented centered FP64 reference. It is not an internal FA3 kernel. No BF16-floor acceptance is implied.}
\label{tab:gproj-accuracy}
\begin{tabular}{lrrr}
\toprule
Method & $dQ$ (\%) & $dK$ (\%) & $dV$ (\%) \\
\midrule
FA3 & 773 & 85.6 & 0.369 \\
QY-shift (source-faithful)$^\ddagger$ & 1307 & 264 & 80.5 \\
FA3-SBS & 219 & 13.3 & 0.334 \\
GProj-Q & 0.342 & 13.3 & 0.334 \\
GProj & 0.342 & 0.371 & 0.334 \\
\midrule
FP32 attention (eager)$^\dagger$ & 0.384 & 0.371 & 0.165 \\
FP32 attention (fused)$^\dagger$ & 0.385 & 0.371 & 0.165 \\
\bottomrule
\end{tabular}
\end{table}
}
\newcommand{\GProjTrainingTable}{%
\begin{table}[t]
\centering\small
\caption{Complete training-step time on one H200, batch one, 4096 tokens. Both panels use the median of three round means, with 15 distinct measured updates per round. Overhead uses only the FA3 row in the same panel. Checkpoint-8000 weights and learning-rate phase, fresh optimizer state; no inter-rank communication. Memory is the maximum allocated GiB over three workers. Panels are separate jobs/nodes; no cross-panel speed ratio is claimed.}
\label{tab:gproj-training}
\begin{tabular}{lrrr}
\toprule
Method & ms/update & vs. FA3 & peak GiB \\
\midrule
\multicolumn{4}{l}{\textit{(1) GProj comparison}} \\
FA3 & 227.357 & +0.00\% & 9.191 \\
FP32 attention (eager) & 719.420 & +216.43\% & 9.182 \\
GProj & 238.120 & +4.73\% & 9.191 \\
\midrule
\multicolumn{4}{l}{\textit{(2) optimized FP32 comparison}} \\
FA3 & 231.031 & +0.00\% & 9.191 \\
FP32 attention (eager) & 734.841 & +218.07\% & 9.182 \\
FP32 attention (fused) & 308.276 & +33.43\% & 9.191 \\
\bottomrule
\end{tabular}
\end{table}
}
\newcommand{\GProjPerCaptureTable}{%
\begin{table}[t]
\centering\small
\caption{Per-capture relative $L_2$ errors (\%) for the complete GProj implementation and fused FP32 baseline. Different FP64 evaluation implementations are disclosed in the text. Captures come from the from-scratch FA3 run and are not independent training replicates.}
\label{tab:gproj-per-capture}
\begin{tabular}{rrrrrrr}
\toprule
Capture & GProj $dQ$ & $dK$ & $dV$ & FP32 $dQ$ & $dK$ & $dV$ \\
\midrule
0 & 0.3223 & 0.2942 & 0.1699 & 0.4203 & 0.3290 & 0.1592 \\
1 & 0.2504 & 0.2397 & 0.1924 & 0.2245 & 0.2216 & 0.1648 \\
2 & 0.2669 & 0.3183 & 0.1734 & 0.2799 & 0.3411 & 0.1603 \\
3 & 0.2668 & 0.2855 & 0.1966 & 0.4146 & 0.3760 & 0.1639 \\
4 & 1.2604 & 1.8577 & 1.0888 & 1.0951 & 2.7637 & 0.1771 \\
5 & 0.4158 & 0.5350 & 0.4715 & 0.3552 & 0.3657 & 0.1648 \\
6 & 0.9490 & 1.3361 & 1.1936 & 0.9065 & 2.5611 & 0.1745 \\
7 & 0.3612 & 0.4240 & 0.5267 & 0.3187 & 0.3764 & 0.1673 \\
\bottomrule
\end{tabular}
\end{table}
}

\newcommand{\ScratchQyTokens}{13.2B}

\newcommand{\RfourteenCaption}[2]{\begingroup
\let\RfourteenOriginalCaption\caption
\renewcommand{\caption}[1]{\RfourteenOriginalCaption{#1}}#2\endgroup}
\let\RfourteenAccuracy\GProjAccuracyTable
\renewcommand{\GProjAccuracyTable}{\RfourteenCaption{Full-tensor relative
$L_2$ error (\%) of BF16 gradients against an FP64 reference, median over
eight captured attention inputs of the 450M model. FA3-SBS and the GProj rows
share identical forward outputs. $^{\dagger,\ddagger}$See
Appendix~\ref{app:gproj-experiments}.}{\RfourteenAccuracy}}
\let\RfourteenTraining\GProjTrainingTable
\renewcommand{\GProjTrainingTable}{\RfourteenCaption{Complete training-step
time and peak memory; overhead is relative to the FA3 row of the same panel.
Measurement protocol in Appendix~\ref{app:gproj-experiments}.}{\RfourteenTraining}}
\let\RfourteenPerCapture\GProjPerCaptureTable
\renewcommand{\GProjPerCaptureTable}{\RfourteenCaption{Per-capture relative
$L_2$ errors (\%) of GProj and fused FP32 attention against FP64 references;
Table~\ref{tab:gproj-capture-map} identifies the captures.}{\RfourteenPerCapture}}
\definecolor{gprojnew}{HTML}{C75B00}
\definecolor{fathree}{HTML}{2F4B6E}\definecolor{fathreefill}{HTML}{EEF2F7}
\definecolor{castred}{HTML}{C62828}\definecolor{castredfill}{HTML}{FBEAEA}
\definecolor{gproj}{HTML}{C75B00}\definecolor{gprojfill}{HTML}{FFF3E6}\definecolor{gprojbg}{HTML}{FFF9F2}
\newcommand{\gpnew}[1]{\leavevmode{\color{gprojnew}#1}}

\begin{document}
\maketitle
\begin{abstract}
BF16 is now standard in large-scale pretraining, including in fused attention
kernels such as FlashAttention, and these kernels are widely trusted. When we
used FlashAttention-3 to pretrain a 450M-parameter transformer on 50B tokens,
however, we ran into a problem: training was healthy for 25B tokens, then the
gradient norm grew a thousandfold and the loss ended 0.2 nats above FP32
attention, without a single NaN. Recomputing the
attention backward of just two layers in FP32 removes almost all of the excess
gradient. Part of the cause is known: a fused multiply-add in the forward
softmax, so far treated as an extreme-input NaN case and never fixed in
FlashAttention-3. Repairing it stops the blow-up, but the query gradient is
still wrong by more than its own size, and training still drives attention
logits to thousands of times their size under accurate gradients. The remaining
error comes from a broken conservation law. The softmax score gradient sums to
zero along every row, which makes the query gradient blind to where the keys sit
as a group; rounding it to BF16 leaves a small nonzero sum that leaks the mean
key into the gradient, and the leak grows exactly as late training makes keys
large and attention sharp. We introduce GProj (gauge projection), which restores the zero sum after
the cast with two rank-one corrections per row. It cuts the remaining median query/key
gradient errors from 219\%/13\% to 0.34\%/0.37\%, on par with FP32 attention,
for 4.7\% more time per training step. In matched from-scratch runs it trains
to the same loss as FP32 attention, while FlashAttention-3 and key smoothing
both destabilize.
\end{abstract}

\section{Introduction}
\label{sec:introduction}
\begin{figure}[!b]
\centering
\includegraphics[width=\linewidth]{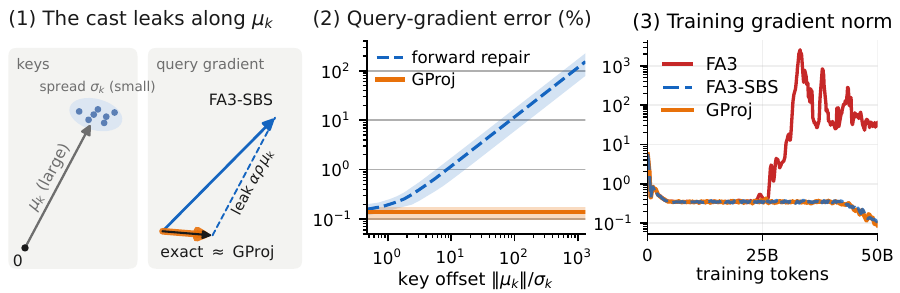}
\vspace{-1.6em}
\caption{(1) The exact score gradient sums to zero along each row; its BF16 cast $t$ leaves a
row mass $\rho\neq0$, which adds a leak $\alpha\rho\mu_k$ along the mean key to the query gradient,
while GProj removes $\rho$ after the cast. (2) Synthetic rows with an exact forward and only the
score gradient cast to BF16: the error with a perfect forward grows with the key offset, so no forward
repair can remove it, while GProj stays at the BF16 floor (median and interquartile range over 256
rows; Appendix~\ref{app:gproj-theory}). (3) FA3 blows up late in training; GProj does not
(Section~\ref{sec:experiments}).}
\label{fig:teaser}
\end{figure}
Low-precision arithmetic is what makes large-scale pretraining affordable
\citep{micikevicius2018mixed,kalamkar2019bfloat16}. Frontier and fully open
models are pretrained with BF16 or FP8 matrix products over trillions of tokens
\citep{grattafiori2024llama3,deepseek2024v3,kimi2025k2,k2horizon2026}, with
attention computed by fused kernels such as FlashAttention
\citep{dao2022flashattention,dao2024flashattention2,shah2024flashattention3}
that keep only BF16 operands and a few FP32 statistics. These kernels are
validated against a higher-precision reference on random, well-conditioned
inputs and then trusted in training. The trust is rarely tested where it
matters, because a wrong gradient does not announce itself: it causes no crash
and no NaN, and the worse model it produces is easily blamed on hyperparameters
or data. Instabilities that do surface are usually traced to the model or the
optimizer, such as growing attention logits \citep{dehghani2023vit22b,wortsman2023smallscale},
attention entropy collapse \citep{zhai2023entropy} or outlier activations
\citep{fishman2024fp8}; the numerical fidelity of the attention kernel itself is
rarely examined \citep{golden2024stable}.

We ran into such a case while pretraining a 450M-parameter transformer on
50B tokens with FlashAttention-3 (FA3). For the first 25B tokens, training was
healthy. Then the gradient norm grew a thousandfold, and the loss drifted up and
never recovered, ending 0.2 nats above an otherwise identical run with FP32
attention (Section~\ref{sec:experiments}). Nothing overflowed and
nothing became NaN. To find the source, we recomputed the attention backward of
the two most affected layers in FP32 while keeping every forward activation
bitwise identical; the excess gradient almost disappeared
(Section~\ref{sec:training-observation}). The fault was in the attention backward, which
raises the question this paper answers: how can a kernel whose outputs are
accurate return gradients that are this wrong?

Part of the answer is already known. The community has traced a BF16 NaN in
attention to a fused multiply-add in the forward softmax: to save an
instruction, ``scale'' and ``subtract the row maximum'' are merged into one
rounded operation, so the largest score no longer maps exactly to zero
\citep{pytorch2024fma}. Because the report involved extreme logits, the fix was
reasonably treated as a corner case: it became an option in FlashAttention-2,
off by default, and never reached FA3 \citep{flashattn2024unfusefma}. Related
remedies followed, from a dynamic softmax shift \citep{qiu2026lowprecision} to
keeping the attention output in FP32 \citep{kimi2026k3}. When we repair the
forward in the same spirit, subtracting the maximum before scaling (FA3-SBS),
the thousandfold growth disappears and training stays stable to the end. The
known fix seemed to be enough.

It was not. With the forward repaired, the query gradients of the affected
layers are still wrong by more than their own size (Section~\ref{sec:experiments}),
and the run still drives these layers into extreme attention: by 33.6B tokens
their typical attention score is about $4$--$5\times10^4$, against under 10 in
the same model trained with accurate gradients (Section~\ref{sec:two-layers}).
Stable training, it turns out, is not evidence of correct gradients, and the
remaining error has a different origin, one that no forward repair can reach.

That origin is a broken symmetry. The gradient with respect to a query should
depend only on how the keys differ from one another, never on where they sit as
a group. Every exact attention backward enforces this translation symmetry
through a conservation law: the softmax score gradient sums to zero along each
row, and that zero sum is what cancels the common part of the keys. FA3 rounds
this score gradient to BF16 before multiplying it with the keys. The rounding
leaves a small nonzero row sum, and the product turns it into a spurious term
proportional to the mean key, which we call the leak (Figure~\ref{fig:teaser}, panel 1).
The leak grows with how far the keys sit from the origin relative to their spread
(panel 2), and when attention also concentrates on a single key, the true gradient
is nearly zero, so the leak outweighs the signal by orders of magnitude. Because the
error is created after the forward pass has finished, a four-key example with an
exact forward already exhibits it (Section~\ref{sec:channel-two}).

This view also explains why the existing remedies stop halfway. The FMA fix,
FP32 outputs and dynamic shifting all act on the forward pass. Key smoothing
\citep{zhang2026sagebwd} subtracts the average key before the kernel, which
shrinks the mean key the leak is multiplied by but leaves the broken sum intact.
When attention concentrates on a single key, the mean that matters is that key
itself, which smoothing does not remove: in our runs, key smoothing delays the
blow-up by about 3B tokens but does not prevent it.

The natural fix is therefore to restore the law itself. GProj (gauge projection) puts the zero row
sum back after the cast, using the actual mass of the BF16 probabilities the
kernel multiplies, with two rank-one corrections per row: one inside the existing backward pass and one
in a second pass for the key gradient. It brings the query and key gradient errors down to the level of FP32
attention, for 4.7\% more time per training step. Trained from
scratch, it stays stable throughout and ends at the same loss as FP32 attention.

In summary, this paper makes three contributions:\par\nopagebreak
\begin{itemize}
\item[$\bullet$] We identify a silent late-training failure of BF16
FlashAttention-3 and trace it to the attention backward.
\item[$\bullet$] We explain it through a conservation law of the attention
gradient that the BF16 cast breaks, an error no forward repair can remove.
\item[$\bullet$] We propose GProj, which restores the law and matches FP32
attention in gradient accuracy and final loss at 4.7\% extra cost.
\end{itemize}

\section{Related Work}
\label{sec:related-work}
\paragraph{Low-precision training and attention kernels.}
Mixed-precision and BF16 training \citep{micikevicius2018mixed,kalamkar2019bfloat16}
and, more recently, FP8 training \citep{peng2023fp8lm,fishman2024fp8} trade
precision for throughput; rounding can bias updates, which motivated stochastic
rounding and careful BF16 recipes \citep{gupta2015limited,zamirai2020bf16}.
FlashAttention and its successors reduce attention memory traffic
\citep{dao2022flashattention,dao2024flashattention2,shah2024flashattention3}, and
quantized attention pushes precision further down
\citep{zhang2025sageattention2,zhang2025sageattention,cheng2025pasa}.
\citet{golden2024stable} quantify the numerical deviation FlashAttention introduces
during training, mainly through its forward outputs; we find a backward error that
can dominate the gradient. The FMA cancellation hazard behind FA3-SBS is known \citep{pytorch2024fma};
FA3-SBS repairs the saved output, not the score-gradient cast.

\paragraph{Softmax reformulations.}
Beyond classical shifted-softmax analysis \citep{blanchard2021softmax,pebay2008covariance},
\citet{qiu2026lowprecision} connect saved-output error to the backward
reduction and biased updates, motivating conditional dynamic shifting.
Kimi K3 adopts the related remedy of keeping
the attention output in FP32 during training \citep[Section~2.1]{kimi2026k3}.
Both address only the saved-output channel.

\paragraph{Conservation-aware low-precision kernels.}
SageBwd uses the exact zero row sum of the score gradient to motivate smoothing keys
before quantization \citep[Sections~4.2 and~6]{zhang2026sagebwd}.
Direct-P matches normalization to consumed probabilities for FP4
\citep[Section~4.3]{hu2026fp4}, and MXAttention normalizes by the mass of the
quantized exponentials \citep{yu2026mxattention}. GProj instead measures the mass after the BF16
cast and projects the contractions with $\rho/m$ (Section~\ref{sec:method}).

\paragraph{Training-stability interventions.}
Growing attention logits are a known source of instability, addressed by
query--key normalization \citep{henry2020qknorm,dehghani2023vit22b,wortsman2023smallscale},
entropy control \citep{zhai2023entropy}, or direct rescaling of query and key
weights during optimization \citep{kimi2025k2,liu2025muonscalable,heo2021adamp}.
These methods change the model or the optimizer; we show that the logit growth
itself can be driven by the kernel's gradient error, which GProj repairs.

\section{Where the Gradient Error Comes From}
\label{sec:training-observation}
\label{sec:mechanism}
\label{sec:diagnosis}
A wrong gradient is hard to see, because nothing in the forward pass changes.
In the FA3 run of Figure~\ref{fig:scratch-progress}, a 450M-parameter
transformer (Appendix~\ref{app:model-configuration}), loss and gradient norm
look normal for 25B tokens; the gradient norm then passes 10 by 30B tokens and
the loss degrades, with no NaN at any point (Figure~\ref{fig:main-diagnosis}).
This section first shows that the excess gradient is a numerical error of the
attention backward, and then derives where that error comes from.

\subsection{The excess gradient is a numerical error}
\label{sec:two-layers}
The failure is confined to two layers. We measure the typical score magnitude
of a layer by its logit scale $\alpha\,\mathrm{rms}(Q)\,\mathrm{rms}(K)\sqrt{d}$,
with $\alpha$ the attention scale and $d$ the head dimension. It stays moderate
in 22 of the 24 layers but grows by orders of magnitude in layers 5 and 11,
where queries and keys become large and almost every query puts nearly all of
its attention on a single key, the near one-hot pattern also seen in
attention sinks and massive activations
\citep{xiao2024streaming,sun2024massive,gu2025sink}.

Large logits alone do not make a gradient wrong, so we separate the two
directly. We freeze the model at 33.6B tokens and recompute only the attention
backward of these two layers in FP32, leaving every attention output and the
loss bitwise identical. The model-gradient norm falls by more than two orders
of magnitude (Figure~\ref{fig:main-diagnosis}, panel 3): almost all of the
observed gradient is numerical error from two attention backward passes.

The error and the large-logit regime feed each other. In the matched
from-scratch runs (Section~\ref{sec:experiments}), which change only the
attention computation, every variant whose backward keeps the BF16 error,
including the forward repair of Section~\ref{sec:channel-one} and key
smoothing, drives layers 5 and 11 into this regime, while FP32 attention and
GProj stay out of it (Table~\ref{tab:regime-audit}). Since FA3-SBS and GProj differ only in the
backward, the gradient error itself drives the logits up; larger logits in turn
come with larger keys, which amplify the error (Theorem~\ref{thm:fwd-vs-proj}).

\subsection{A conservation law of the attention backward}
\label{sec:setting}
For one query $q\in\R^d$, let $L$ be its causal support, $k_j\in\R^d$
and $v_j\in\R^{d_v}$ the keys and values, $\alpha>0$ the attention scale, and
$u=\partial\mathcal L/\partial o$ the incoming derivative. Define
\begin{equation}
 \begin{gathered}
 s_j=\alpha q^\top k_j,\quad
 p_j=\frac{e^{s_j}}{\sum_{\ell\in L}e^{s_\ell}},\quad o=\sum_jp_jv_j,\\
 a_j=u^\top v_j,\quad \mu_a=\sum_jp_ja_j,\quad \mu_k=\sum_jp_jk_j,\quad g_j=p_j(a_j-\mu_a).
 \end{gathered}
 \label{eq:row-attention}
\end{equation}
Sums run over $L$. The score derivative $g$ is one row of $dS$, and the
vector--Jacobian product (VJP) is
\begin{equation}
 G_Q=dq=\alpha\sum_jg_jk_j,\qquad
 G_{K,j}=dk_j=\alpha g_jq,\qquad G_{V,j}=dv_j=p_ju.
 \label{eq:row-gradients}
\end{equation}
Capital $Q,K,V,U$ collect rows, and $dQ,dK,dV$ are full gradient tensors;
key and value contributions also sum over queries and over the query heads that
share a KV head in grouped-query attention (GQA). We measure errors against this
VJP in FP64 on the same BF16 inputs.

The exact score gradient has a conserved quantity: its \emph{row mass}, the sum
of its entries, is zero. This is what makes attention gradients insensitive to
where the keys sit as a group.
\begin{proposition}[Exact translation structure]
\label{prop:main-translation}
At fixed $q,u$, common translations of unmasked keys or values preserve the
local VJP; $\mathbf1^\top g=0$ and
$G_Q=\alpha\sum_jp_j(k_j-\mu_k)(a_j-\mu_a)$.
\end{proposition}
The proof is one line: $\sum_jg_j=\sum_jp_ja_j-\mu_a\sum_jp_j=0$, so shifting
every key by $\mu_k$ leaves $\sum_jg_jk_j$ unchanged, which gives the covariance
form (Appendix~\ref{app:leakage-details}). A numerical backward can break this
law in two places: before the score gradient is formed, through the quantities
it is built from, or after, when it is rounded and multiplied with the keys. We
call these channels one and two and examine them in turn.

\begin{figure}[t]
\centering
\includegraphics[width=\linewidth]{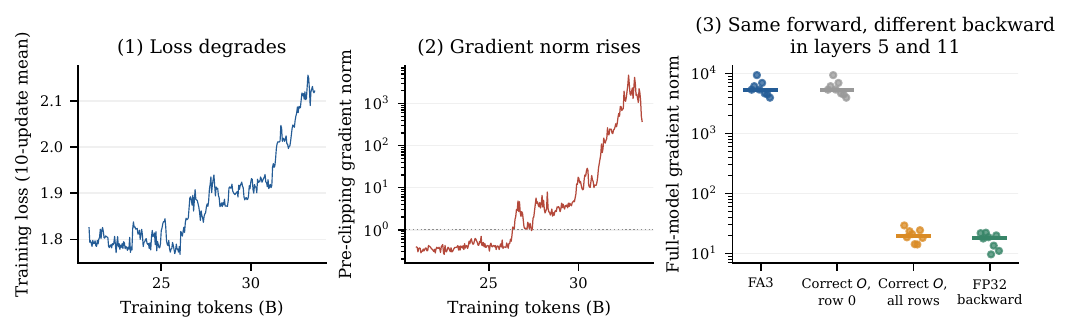}
\caption{The failure and its localization. (1)--(2) Loss and pre-clipping
gradient norm of the FA3 run of Figure~\ref{fig:scratch-progress} (10-update means). (3) Model-gradient norm on eight documents
at 33.6B tokens, with the forward held bitwise fixed and only the backward of
layers 5 and 11 changed: FA3's own; FA3 given a correct saved output (``correct
$O$'', recomputed in FP64 and rounded to BF16) for the first query row or for all
rows (removes channel one, Section~\ref{sec:channel-one}); or FP32 (removes both
channels).}
\label{fig:main-diagnosis}
\end{figure}

\subsection{Channel one: the saved output}
\label{sec:channel-one}
FA3's backward does not recompute $\mu_a=u^\top o$; it estimates it as
$\widehat\delta=u^\top\widehat o$ (the row statistic $D$ in
Algorithm~\ref{alg:gproj-main}) from the BF16 output $\widehat o$ saved by the
forward pass. An error $e_\delta=\widehat\delta-\mu_a$ in this single reduction
shifts the score gradient by $-e_\delta p$ and produces the query-gradient error
\begin{equation}
 E_\delta=-\alpha e_\delta\mu_k,
 \label{eq:main-delta}
\end{equation}
a vector along the attention-weighted mean key. This one-line prediction
matches the measured error almost exactly, and supplying a correct saved output
removes most of it, together with nearly all of the excess model gradient
(Figure~\ref{fig:main-diagnosis}; Appendix~\ref{sec:historical-numerics}). The saved-output error has a known source: a
fused multiply-add in the forward softmax merges scaling and maximum subtraction
into one rounded operation, so the largest score no longer maps exactly to zero
\citep{pytorch2024fma}. Subtracting the unscaled row maximum before scaling, a
one-line change we call FA3-SBS (subtract before scale), removes it and
substantially lowers the gradient errors (Section~\ref{sec:experiments}).

\subsection{Channel two: the cast breaks the conservation law}
\label{sec:channel-two}
Removing channel one repairs the model-gradient norm but not the query
gradient. With a correct saved output the norm returns to its FP32 level, yet
the query gradient is still wrong by more than its own size and poorly aligned
with the exact one (Table~\ref{tab:historical-confirmation}), and this
residual still drives layers 5 and 11 into the large-logit regime. It enters
after the score gradient is formed: FA3 rounds $g$ to a BF16 operand $t$ before
contracting it with the keys, and rounding need not preserve a zero sum. With
$e=t-g$ and $\rho=\sum_jt_j$, the contraction satisfies, before its own rounding,
\begin{equation}
 \alpha\sum_j t_jk_j-G_Q
 =\alpha\sum_j e_j(k_j-\mu_k)+\alpha\rho\mu_k.
 \label{eq:leakage-decomposition}
\end{equation}
The first term is an ordinary rounding error, measured relative to the mean
key. The second is a leak: the row mass times the mean key itself. It is
harmless while keys are small and attention is spread out, and it dominates late
in training. When a row puts almost all of its weight on one key, the true
gradient is a covariance over a nearly one-hot distribution and is close to
zero, whereas $\|\mu_k\|$ is close to the norm of that key, which late training
makes large; a row mass of a single BF16 rounding step then outweighs the signal
by orders of magnitude. Two keys make this explicit. With weights $1-\varepsilon$
and $\varepsilon$,
\begin{equation}
 G_Q=\alpha\,\varepsilon(1-\varepsilon)(a_1-a_2)(k_1-k_2),\qquad
 |\rho|\le 2u_b\,\varepsilon(1-\varepsilon)|a_1-a_2|,
 \label{eq:two-key}
\end{equation}
where $u_b=2^{-8}$ is the BF16 unit roundoff, so the rounding model allows a
leak of up to $2u_b\|\mu_k\|/\|k_1-k_2\|$ relative to the true gradient. In
this bound the sharpness $\varepsilon$ cancels: concentrating attention does not
protect the gradient, while every increase in key norm relative to key
separation enlarges the leak linearly. The four-key example below realizes a
leak of this kind under actual rounding.
Neither FP32
accumulation nor a forward repair can remove the leak: it is already in the BF16
multiplicands, and it arises after the forward pass. The key gradient
$dk_j=\alpha\sum_it_{ij}q_i$ contracts the same operand along the queries and
inherits its row-mass error.

\paragraph{A four-key example.} Take $q=0$, so attention is uniform over four keys
$k_j=(65536,\,j-1)$ that share their first coordinate; the first coordinate of
the true query gradient is therefore exactly zero. With $\alpha=1/8$ and values
chosen so that $g=(1+\eta,-1,-\eta,0)$, where $\eta=3/1024$, every
input and the saved output are exact in BF16. Casting $g$ to BF16 gives
$t=(1,-1,-\eta,0)$: the entry $1+\eta$ rounds to 1, leaving row mass
$\rho=-\eta$. The first coordinate of the query gradient becomes
$\alpha\rho\cdot65536=-24$ instead of 0 (Appendix~\ref{app:leakage-details}
lists the values).

\section{GProj}
\label{sec:method}
To remove the row-mass leak of the BF16 score-gradient cast
(Section~\ref{sec:channel-two}), GProj restores the zero row sum after the
cast, on the operands the kernel actually multiplies. It removes the row mass $\rho$ that the cast leaves in $t$ by subtracting a
multiple of the BF16 probabilities $r$, an operand the kernel already forms for
$dV$, whose actual mass is $m=\sum_jr_j$. Subtracting $\lambda r$ leaves row mass
$\rho-\lambda m$, so cancellation requires $\lambda=\rho/m$; using $\rho$ alone
leaves $\rho(1-m)$, because BF16 probabilities need not sum to one. The result,
$t^\perp=t-\lambda r$, is the projection of $t$ onto the zero-sum subspace along
$r$. It never has to be formed, because both contractions are linear in it:
\begin{equation*}
 dQ=\alpha\,t^\perp K=\alpha\big(tK-\lambda\odot(rK)\big),\qquad
 dK=\alpha\,{t^\perp}^{\!\top}Q=\alpha\big(t^\top Q-r^\top(\lambda\odot Q)\big).
\end{equation*}
FA3 already computes $tK$ and $t^\top Q$, so GProj adds, per gradient, one
product ($rK$, and $r^\top(\lambda\odot Q)$ in a second pass) and one rank-one
term per query row, plus two row sums, keeping BF16 products, FP32
accumulators, and the existing saved BF16 output and FP32 LSE. The query
correction fits in FA3's own backward pass. The key correction needs a second
pass: $\lambda$ is defined per query row and is known only after that row has
seen every key, whereas $dK$ is accumulated per key tile over all query rows.
GProj keeps the subtract-before-scale forward of FA3-SBS. Algorithm~\ref{alg:gproj-main}
shows where GProj enters FA3's backward, and kernel details
are in Appendix~\ref{app:implemented-arithmetic}.

\begingroup
\footnotesize
\begin{algorithm}[t]
\caption{FA3 backward pass (black) with the GProj additions (\gpnew{orange})}
\label{alg:gproj-main}
\centering
\begin{tabular}{@{}r@{\hspace{.7em}}p{\dimexpr\linewidth-2.4em\relax}@{}}
\toprule
& \textbf{Input:} BF16 $Q,K,V,O,dO\in\R^{N\times d}$; FP32 $\mathrm{LSE}\in\R^N$; scale $\alpha$; key tiles $K_j,V_j$, query tiles $Q_i,O_i,dO_i$.\\
& \textbf{Output:} BF16 $dQ,dK,dV$. Accumulators and row statistics are FP32.\\
\midrule
1 & $D\gets\mathrm{rowsum}(dO\odot O)$;\quad $dQ^{\rm acc}\gets0$;\quad \gpnew{$B^{\rm acc}\gets0$,\ $\rho\gets0$,\ $m\gets0$}\\
2 & \textbf{for} each key tile $j$ \textbf{in parallel do}\\
3 & \quad Load $K_j,V_j$;\quad $dK^{\rm acc}_j\gets0$,\ $dV_j\gets0$\\
4 & \quad \textbf{for} each query tile $i$ that attends to tile $j$ \textbf{do}\\
5 & \qquad Load $Q_i,dO_i,\mathrm{LSE}_i,D_i$;\quad $S\gets\alpha\,Q_iK_j^\top$;\quad $P\gets\exp(S-\mathrm{LSE}_i)$ (masked)\\
6 & \qquad $dV_j\mathrel{+}=\mathrm{BF16}(P)^\top dO_i$;\quad $dP\gets dO_iV_j^\top$\\
7 & \qquad $t\gets\mathrm{BF16}\big(P\odot(dP-D_i)\big)$\\
8 & \qquad \gpnew{$r\gets\mathrm{BF16}(P)$;\quad $\rho_i\mathrel{+}=\mathrm{rowsum}(t)$;\quad $m_i\mathrel{+}=\mathrm{rowsum}(r)$}\\
9 & \qquad $dQ^{\rm acc}_i\mathrel{+}=t\,K_j$;\quad \gpnew{$B^{\rm acc}_i\mathrel{+}=r\,K_j$}\hfill(atomic adds)\\
10 & \qquad $dK^{\rm acc}_j\mathrel{+}=t^\top Q_i$\\
11 & \quad \textbf{end for}\\
12 & \quad Write $dV_j$;\quad FA3 writes $dK_j\gets\mathrm{BF16}(\alpha\,dK^{\rm acc}_j)$, \gpnew{GProj keeps the FP32 $dK^{\rm acc}_j$ for Pass 2}\\
13 & \textbf{end for}\\
14 & \gpnew{$\lambda\gets\rho/m$ (zero where $m=0$)};\quad $dQ\gets\mathrm{BF16}\big(\alpha(dQ^{\rm acc}\gpnew{\,-\,\lambda\odot B^{\rm acc}})\big)$\\
\midrule
15 & \gpnew{\textbf{Pass 2: for} each key tile $j$ \textbf{in parallel do}\quad $C\gets0$}\\
16 & \quad\gpnew{\textbf{for} each query tile $i$ that attends to tile $j$ \textbf{do}}\\
17 & \qquad\gpnew{Recompute $P$ as in line 5;\quad $\widetilde r\gets\mathrm{BF16}(P)$;\quad $C\mathrel{+}=\widetilde r^\top\,\mathrm{BF16}(\lambda_i\odot Q_i)$}\\
18 & \quad\gpnew{\textbf{end for};\quad $dK_j\gets\mathrm{BF16}\big(\alpha(dK^{\rm acc}_j-C)\big)$}\\
\bottomrule
\end{tabular}
\end{algorithm}
\endgroup

\label{sec:theory}
\label{sec:gproj-projection}
The projection works because the exact score gradient already has zero row
sum: since $\mathbf1^\top g=0$, the projection leaves $g$ unchanged and acts
only on the cast error $e=t-g$, removing its row mass. In the query contraction this
subtracts the $r$-weighted center $\mu_r$ from every key, and a common key
offset cancels exactly. Proposition~\ref{prop:gproj-projection} states this
precisely; proofs for this section are in Appendix~\ref{app:gproj-theory}.
\begin{proposition}[Projection with consistent operand mass]
\label{prop:gproj-projection}
For nonnegative $r$ with $m>0$, define
$\Pi_r=I-r\mathbf1^\top/m$ and $\lambda=\rho/m$.
Then $\Pi_r^2=\Pi_r$, $\mathbf1^\top\Pi_r=0$, and
$t^\perp=\Pi_rt=t-\lambda r$ has zero row sum. With
$\mu_r=\sum_jr_jk_j/m$, its query contraction is
\begin{equation}
 G_Q^{\rm proj}
   =\alpha\left(\sum_jt_jk_j-\frac{\rho}{m}\sum_jr_jk_j\right)
   =\alpha\sum_jt_j(k_j-\mu_r).
 \label{eq:gproj-query}
\end{equation}
For fixed $t,r$, this contraction is unchanged by any common translation
of the keys. If $e=t-g$ for the exact score gradient $g$, then
\begin{equation}
 t^\perp-g=\Pi_re,\qquad
 G_Q^{\rm proj}-G_Q=\alpha\sum_je_j(k_j-\mu_r).
 \label{eq:gproj-centered-error}
\end{equation}
\end{proposition}
It is also the smallest zero-sum correction in the $r$-weighted norm:
$t-\lambda r$ minimizes $\sum_j(z_j-t_j)^2/r_j$ over all $z$ with $\mathbf1^\top z=0$.

The remaining error is therefore bounded by the spread of the keys, not by
their offset; a backward that contracts $t$ directly has no such bound. To compare
the two on equal terms, give every method a perfect forward: exact
probabilities $p$ and the exact center $\mu_a$, so the pre-cast score operand is
exactly $x=g$. Only the BF16 cast remains, $t_j=x_j+\xi_j$ with
$|\xi_j|\le u_b|x_j|$ for the unit roundoff $u_b=2^{-8}$. Write
$\sigma_k^2=\sum_jp_j\|k_j-\mu_k\|^2$ and $\sigma_a^2=\sum_jp_j(a_j-\mu_a)^2$
for the attention-weighted spread of keys and of $a$.

\begin{theorem}[Forward repair versus projection]
\label{thm:fwd-vs-proj}
Under these assumptions, with exact arithmetic after the cast:
\begin{enumerate}
\item[(a)] A backward that contracts $t$ directly, as FA3-SBS, FP32 saved
outputs and dynamic shifting do, has query error
$E_{\rm fwd}=\alpha\sum_j\xi_j(k_j-\mu_k)+\alpha\big(\sum_j\xi_j\big)\mu_k$.
Over cast errors allowed by the rounding model,
\[
 \sup_\xi\|E_{\rm fwd}\|\;\ge\;|\alpha|\,u_b\Big(\|\mu_k\|\sum_j|g_j|-\sigma_k\sigma_a\Big).
\]
Translating all keys by $b$ leaves $G_Q$ unchanged but replaces $\mu_k$ by
$\mu_k+b$, so for any fixed $G_Q\neq0$ the worst-case relative error is
unbounded.
\item[(b)] GProj with $r=p$ has query error
$E_{\rm proj}=\alpha\sum_j\xi_j(k_j-\mu_k)$ and
$\|E_{\rm proj}\|\le|\alpha|\,u_b\,\sigma_k\sigma_a$ for every admissible cast;
the bound does not depend on $\mu_k$ and is invariant under key translation.
\end{enumerate}
\end{theorem}

The two bounds differ by the factor $\|\mu_k\|\sum_j|g_j|/(\sigma_k\sigma_a)$:
how far the keys sit from the origin relative to their spread ($2\|\mu_k\|/\|k_1-k_2\|$
for two keys, at any split), which late training makes large; Figure~\ref{fig:teaser} (panel 2)
measures this growth on synthetic rows. This model also leaves out one effect of
concentration. In the implemented kernel, the operand is formed
from a BF16 saved output and reconstructed probabilities, so its row mass
carries an error of order $u_b$ that does not shrink with the remaining mass,
whereas the true gradient does (Section~\ref{sec:channel-one}). GProj removes
this row mass whatever its source. Appendix~\ref{app:gproj-theory} gives the proof
and the additional terms of the implemented kernel, where $r$ is itself a BF16
operand and the center is computed from a saved output.

\section{Results}
\label{sec:experiments}
\subsection{Accuracy with an identical forward}
GProj reduces median query/key gradient error from 219\%/13\% after
forward repair to 0.34\%/0.37\% (Table~\ref{tab:gproj-accuracy}). We evaluate
eight attention inputs captured from layers 5 and 11 of the from-scratch FA3
run, and report the median relative $L_2$ error against an FP64 reference on
the same BF16 inputs; on operands from the GProj and FP32 runs, every kernel
is at the BF16 floor (Table~\ref{tab:regime-audit}).

\textbf{FA3} is the unmodified BF16 kernel, \textbf{FA3-SBS} repairs its
forward, \textbf{GProj-Q} adds only the query correction, and \textbf{GProj}
adds both corrections. The last three produce byte-identical forward outputs,
so their comparison isolates the backward. \textbf{FP32 attention} keeps BF16
inputs and outputs but computes attention in FP32. Reference and evaluation
details are in Appendix~\ref{app:gproj-experiments}.

\GProjAccuracyTable

The query correction alone brings $dQ$ to the level of FP32 attention, and the
key pass does the same for $dK$.
GProj does not correct $dV$, whose error is that of FA3-SBS.

\textbf{QY-shift}, the released Qiu--Yao implementation \citep{qiu2026lowprecision}, is less accurate than FA3 on these captures (Appendix~\ref{app:comparison-methods}).

\begin{figure}[t]
\centering
\includegraphics[width=\linewidth]{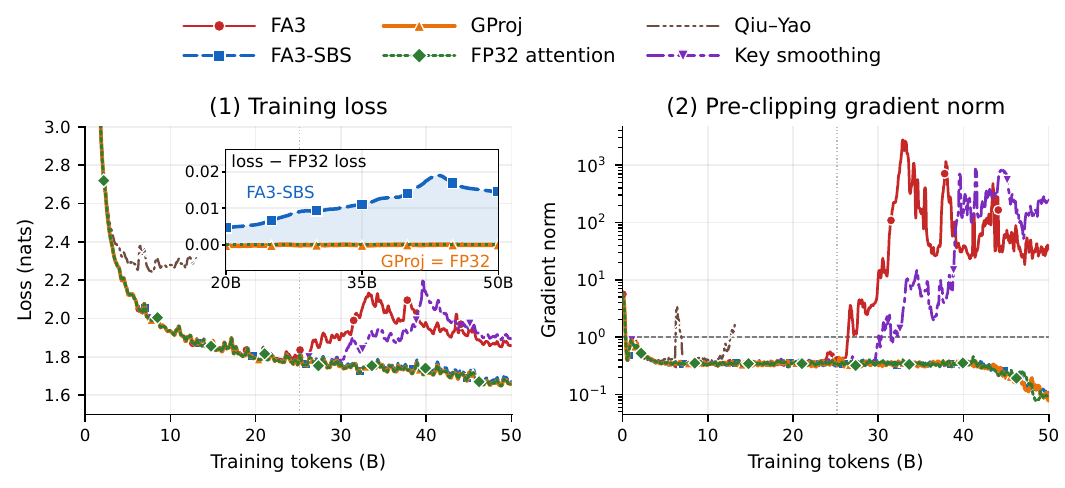}
\caption{Matched from-scratch training of a 450M model with six attention
variants over 50.0B tokens; the Qiu--Yao run was stopped at 13.2B tokens
(Appendix~\ref{app:training-status}). (1) Training loss; the inset shows the loss
of FA3-SBS and GProj minus that of FP32 attention (trailing means over 2.1B tokens):
GProj coincides with FP32, while FA3-SBS drifts 0.01--0.02 nats above it; (2) pre-clipping gradient norm on a log scale, with a
dashed reference at 1. Curves use a light trailing average
over 0.21B tokens; early loss above 3.0 is clipped, and the dotted line at
25.2B tokens marks the approximate onset of FA3's gradient growth.}
\label{fig:scratch-progress}
\end{figure}

\subsection{Translation invariance and robustness}
GProj removes the leak, and each correction changes only the gradient it
targets: the query correction leaves $O,dK,dV$ and the key correction leaves
$O,dQ,dV$ byte-identical to FA3-SBS on every capture and layout tested. In a translation test, one query coordinate is zero and the matching
coordinate of every key is shifted by a common value, so the exact gradient in
that coordinate is zero. GProj returns essentially zero there, whereas FA3 and
FA3-SBS return errors of order one, and on the four-key example of
Section~\ref{sec:channel-two} GProj returns exactly zero.

\begingroup\let\SynOrigCaption\caption
\renewcommand{\caption}[1]{\SynOrigCaption{Synthetic suite: relative $L_2$ error percentages,
median / maximum over the 58 test cases, excluding zero targets separately
for each gradient ($dQ/dK/dV$: 53/51/57 nonzero targets); the two analytic
calibrations are reported in Appendix~\ref{app:synthetic-suite}.}}%
\begin{table}[t]
\centering\small
\caption{Synthetic suite: relative $L_2$ error percentages,
median / maximum over the 58 test cases, excluding zero targets separately
for each gradient ($dQ/dK/dV$: 53/51/57 nonzero targets). The two known
analytic calibrations are excluded from these aggregates but included in
the adverse-result reporting (Appendix~\ref{app:synthetic-suite}).}
\label{tab:synthetic-suite}
\begin{tabular}{lrrr}
\toprule
Method & $dQ$ & $dK$ & $dV$ \\
\midrule
FA3 & 0.2601 / 4794.33 & 0.2537 / 4.1051 & 0.2239 / 0.7657 \\
FA3-SBS & 0.2601 / 4794.53 & 0.2535 / 4.1050 & 0.2239 / 0.7657 \\
GProj & 0.2238 / 2.4938 & 0.2219 / 2.7565 & 0.2239 / 0.7657 \\
\bottomrule
\end{tabular}
\end{table}
\endgroup

Beyond the captures, GProj passes all 60 cases of a synthetic suite covering
common-key shifts, packed, unequal, singleton and tile-boundary layouts, and
long sequences (Table~\ref{tab:synthetic-suite}; Appendix~\ref{app:synthetic-suite}).
Its $dK/dV$ error never rises relative to FA3-SBS; under random common-key translations its error in
the zero-target coordinate is four orders of magnitude below that of FA3-SBS.

\subsection{From-scratch training and cost}
\begin{table}[!t]
\centering\small\setlength{\tabcolsep}{2.5pt}
\caption{Regime statistics at matched checkpoints, measured with exact FP32
attention on two fixed documents. Errors are relative $L_2$ errors of $dQ$
against an FP64 reference on the same BF16 operands, given as ratios rather
than percentages (median of the two documents' means; values above 1 mean the
error exceeds the gradient). With
accurate gradients (GProj, FP32) the layers stay small and FA3's error equals
cuDNN's BF16 floor.}
\label{tab:regime-audit}
\begin{tabular}{lccccc}
\toprule
Run (tokens) & \shortstack{Layers with\\logit scale $>10^3$} & \shortstack{Logit scale\\L5 / L11} & \shortstack{Layer 11\\rms $Q$/$K$} & \shortstack{Layer 11 rows\\$p_{\max}>0.999$} & \shortstack{Layer 11 $dQ$ error\\FA3 / cuDNN}\\
\midrule
FA3 (23.1B) & 5, 11 & 5,436 / 5,519 & 71 / 78 & 86\% & 4.5 / 2.2\\
FA3 (33.6B) & 5, 11 & 12,843 / 51,057 & 213 / 240 & 99.4\% & 253 / 5.6\\
FA3-SBS (33.6B) & 4, 5, 11 & 39,205 / 49,158 & 216 / 227 & 99.6\% & 85 / 2.5\\
Key smoothing (33.6B) & 5, 11 & 10,905 / 17,272 & 124 / 139 & 74\% & 30 / 1.1\\
GProj (33.6B) & none (max 26) & 8 / 7 & 2 / 3 & 0\% & 0.12 / 0.12\\
FP32 attention (33.6B) & none (max 27) & 7 / 7 & 2 / 3 & 0\% & 0.13 / 0.13\\
\bottomrule
\end{tabular}
\end{table}

\GProjTrainingTable

GProj, FP32 attention and FA3-SBS finish the full
50.0B-token schedule with small gradients, while FA3 and key smoothing finish
with higher loss and large gradients (Figure~\ref{fig:scratch-progress};
Table~\ref{tab:scratch-progress}). The released Qiu--Yao implementation stalls
from the first few billion tokens, as its shift rule underflows
(Appendix~\ref{app:training-status}).
GProj ends at the same loss as FP32 attention (1.665), while FA3 ends 0.2 nats
higher. FA3-SBS finishes stably but not cleanly: its backward error still drives
several layers into the large-logit regime, which GProj and FP32 attention
avoid (Table~\ref{tab:regime-audit}).
Key smoothing delays the onset by about 3B tokens but does not prevent it,
consistent with shrinking $\mu_k$ but not $\rho$.

This stability comes at little cost: GProj adds 4.7\% to a complete training
step, against 33.4\% for fused FP32 attention, and saves no extra forward state
(Table~\ref{tab:gproj-training}; Appendix~\ref{app:gproj-experiments}).

\section{Limitations}
\label{sec:limitations}
Our study covers BF16 FlashAttention-3 on Hopper GPUs with head dimension 64
and a 450M-parameter model. The 4.7\% overhead is measured at 4096 tokens;
the GProj backward costs more than FA3's, so the overhead grows with
attention's share of the step. GProj leaves the value gradient to the native
kernel.

\section{Conclusion}
\label{sec:conclusion}
A zero-sum conservation law ties attention's exact gradient to key-translation
invariance, and the BF16 cast breaks it. GProj restores the law after the cast
with two rank-one corrections per row, bringing query/key gradient errors to the level of
FP32 attention, and trains stably where FA3 degrades. Checking exact identities on
the operands a kernel actually multiplies offers a practical way to audit and
repair other low-precision kernels.

\label{main-text-end}
\bibliography{references}
\bibliographystyle{iclr2027_conference}
\clearpage
\appendix
\paragraph{Appendix roadmap.}
Appendices~\ref{app:leakage-details}--\ref{app:gproj-theory} give the
cast witness, implemented arithmetic and proofs; Appendix~\ref{app:comparison-methods}
defines the comparison methods; Appendix~\ref{app:gproj-experiments}
gives the accuracy and cost protocols; Appendix~\ref{app:synthetic-suite}
the synthetic suite; Appendix~\ref{app:training-status} the training runs
and the regime audit. The remaining appendices give the model configuration,
the full localization of the FA3 failure (saved-output channel, fixed-forward
interventions, source intervention, saved-state crosses, backend comparison)
and a plain-Transformer control.
\section{Exact Translation Structure and the Cast Witness}
\label{app:leakage-details}

We use
\emph{gauge} for a common translation of keys or values that leaves the exact
local VJP unchanged at fixed incoming $u$, and the row notation of
Section~\ref{sec:setting} throughout.

\subsection{Exact translations and numerical row-mass leakage}

\begin{proof}[Proof of Proposition~\ref{prop:main-translation}]
For a common key translation $k_j\mapsto k_j+b$, all logits in one row
change by the same $\alpha q^\top b$ and the exact softmax probabilities
are unchanged. For a common value translation $v_j\mapsto v_j+c$, every
$a_j$ changes by $u^\top c$, which cancels in $a_j-\sum_\ell p_\ell a_\ell$.
Thus $g$, and the contributions $\alpha\sum_jg_jk_j$, $\alpha g_jq$ and
$p_ju$, are unchanged because $\sum_jg_j=0$. Expanding the centered
covariance gives $G_Q=\alpha\sum_jp_j(k_j-\mu_k)(a_j-\mu_a)$.
\end{proof}

These identities also underlie key smoothing in SageBwd
\citep[Section~6]{zhang2026sagebwd}; the value translation changes the forward
output by $c$, so $u$ must be held fixed. They single out the direction, the
common key offset, to which a numerical backward can respond spuriously. Key
smoothing removes a common mean from the keys before quantization, which changes
the key operand but not the row sum of the score gradient formed afterwards.
GProj instead corrects the cast score operand itself.

\paragraph{Row-mass leakage.}

\begin{proposition}[Exact leakage decomposition]
\label{prop:cast-mass-leakage}
Let $t$ be any numerical score-gradient operand, let $e=t-g$, and let
$\rho=\sum_jt_j$. Before any additional contraction rounding,
Equation~\eqref{eq:leakage-decomposition} gives the decomposition.
For fixed $t$, a common-key translation $b$ changes its contraction by exactly
$\alpha\rho b$. More generally, subtracting $\lambda r$ from $t$, where
$m=\sum_jr_j\ne0$, leaves row mass $\rho-\lambda m$. Thus $\lambda=\rho/m$
eliminates this mass in real arithmetic, whereas $\lambda=\rho$ leaves
$\rho(1-m)$.
\end{proposition}

\begin{proof}
Since $\sum_jg_j=0$, $\rho=\sum_jt_j=\sum_je_j$. Adding and subtracting
$\mu_k$ inside $\alpha\sum_je_jk_j$ gives
$\alpha\sum_je_jk_j=\alpha\sum_je_j(k_j-\mu_k)+\alpha\rho\mu_k$, which is
Equation~\eqref{eq:leakage-decomposition}. For fixed $t$, translating every key
by $b$ adds $\alpha\rho b$ to the contraction, and subtracting $\lambda r$ from
$t$ leaves row sum $\rho-\lambda m$.
\end{proof}

The last statement is why GProj divides by the actual mass $m$ of the BF16
probabilities it multiplies: the mass of any other probability representation
does not cancel $\rho$.

\paragraph{An exact $-24$ witness.}
A four-key row with an exact forward exhibits the leak. Take four equally likely keys, $\alpha=1/8$, $q=0$, $u=(1,1)$,
$\eta=3/1024$, and
\begin{equation}
 k_j=(65536,j-1),\qquad
 (v_1,v_2,v_3,v_4)=((4,4\eta),(-4,0),(-4\eta,0),(0,0)).
 \label{eq:minus24-witness}
\end{equation}
Every input is exactly BF16-representable. Here $p_j=1/4$ and
$a=(4+4\eta,-4,-4\eta,0)$ has mean zero, so the output is
$o=(-\eta,\eta)$, $u^\top o=0$ exactly, and
$g=(1+\eta,-1,-\eta,0)$. BF16 rounds $1+3/1024$ to one and represents
the other three entries exactly, so $t=(1,-1,-\eta,0)$ with
$\rho=-\eta$. The true first coordinate of
$G_Q$ is zero; contracting the cast operand gives
$-(1/8)\,65536\,(3/1024)=-24$. An exact saved output does not repair this
cast channel. Zero padding embeds the construction in $d=d_v=64$.

\section{Implemented Arithmetic and Rounding}
\label{app:implemented-arithmetic}
The GProj kernel targets Hopper GPUs
with BF16, head dimension 64 and GQA, in dense and packed causal layouts. Matrix
operands, outputs and returned gradients are BF16; accumulators, row statistics
and scalar corrections are FP32. GProj consists of the native backward with a
query tile of 96, followed by a separate $dK$ correction pass.

The projection acts on the BF16 operands after they are formed. Let $\widehat p$ denote the probabilities
reconstructed by the native backward and let $\widehat a_j$ and
$\widehat\delta$ denote its computed value product and saved-output
reduction. The operands consumed by its BF16 matrix multiplications are
\begin{equation}
 r_j=\operatorname{cast}_{\rm BF16}(\widehat p_j),\qquad
 t_j=\operatorname{cast}_{\rm BF16}
       \bigl(\operatorname{fl}_{32}
       [\widehat p_j(\widehat a_j-\widehat\delta)]\bigr),\qquad
 m=\sum_jr_j,\quad \rho=\sum_jt_j.
 \label{eq:gproj-row-mass}
\end{equation}
Here $\operatorname{fl}_{32}$ evaluates the bracketed operations in FP32, and
$m,\rho$ are exact sums of the formed operands, which the kernel approximates in
FP32. The attention scale $\alpha$ is applied later, and in general $m\ne1$ and
$r$ equals neither $p$ nor $\widehat p$.

\paragraph{Two passes.}
The first pass forms $t$ and $r$, measures their row masses, accumulates $tK$
and $rK$, and writes the corrected $dQ$ while keeping the FP32 key accumulator.
The second pass revisits the query/key tiles, reconstructs the probabilities and
corrects that accumulator before $dK$ is cast to BF16.
Algorithm~\ref{alg:gproj-main} gives the tiled pseudocode; the equations below
fix the rounding and the placement of the scale.

\paragraph{Forward and saved state.}
For computed unscaled logits $x_{ij}$ and their row maximum $x_{i,\max}$,
the FA3-SBS forward explicitly rounds the subtraction:
\begin{equation}
 z_{ij}=\operatorname{exp2}_{32}\!\left(
 \operatorname{fl}_{32}[\operatorname{sub}_{32}(x_{ij},x_{i,\max})
 (\alpha\log_2 e)]\right).
 \label{eq:gproj-forward}
\end{equation}
Everything else in the forward is native, and it saves BF16 $O$ and the FP32
log-sum-exp (LSE) as usual.

\paragraph{Query correction.}
After the native preprocessing computes $\widehat\delta_i$ from the saved $O_i$
and $U_i$, the backward forms $t$ and $r$ as in
Equation~\eqref{eq:gproj-row-mass}, measures their masses and accumulates two
contractions:
\begin{equation}
 \widehat A_i=\operatorname{acc}_{32}(t_iK),\quad
 \widehat B_i=\operatorname{acc}_{32}(r_iK),\quad
 \widehat\lambda_i^Q=\operatorname{div}_Q(\widehat\rho_i,\widehat m_i),
 \label{eq:gproj-accumulators}
\end{equation}
where $\operatorname{acc}_{32}$ is a BF16 matrix product accumulated in FP32,
hats denote computed quantities and $\operatorname{div}_Q$ is the compiled FP32
division. After the last key tile,
\begin{equation}
 \widehat{dq}_i=\operatorname{cast}_{\rm BF16}
 \left(\operatorname{fl}_{32}[\alpha(\widehat A_i-
 \widehat\lambda_i^Q\widehat B_i)]\right).
 \label{eq:gproj-query-implemented}
\end{equation}
The scale is applied once, after the correction, and a row with zero computed
mass receives no correction.

\paragraph{Key correction.}
The same projected operand gives the ideal key gradient
$G_K^{\rm proj}=\alpha[t^\top Q-r^\top(\lambda\odot Q)]$, where
$\lambda\odot Q$ scales each query row. Since $\lambda$ is known only after a
row has seen every key, the key correction needs its own pass over the tiles.
It reuses the stored row statistics but divides again, in Triton FP32
arithmetic with the same zero-mass rule, so
$\widehat\lambda_i^K=\operatorname{div}_K(\widehat\rho_i,\widehat m_i)$ need not
match $\widehat\lambda_i^Q$ bitwise. It computes
\begin{equation}
 \widetilde r_{ij}=\operatorname{cast}_{\rm BF16}(\widetilde p_{ij}),\quad
 z_i=\operatorname{cast}_{\rm BF16}(\widehat\lambda_i^Kq_i),\quad
 \widehat{dK}=\operatorname{cast}_{\rm BF16}
 \!\left(\operatorname{fl}_{32}[\widehat A_K-
 \alpha\operatorname{acc}_{32}(\widetilde r^\top z)]\right),
 \label{eq:gproj-key-implemented}
\end{equation}
where $\widehat A_K$ is the native key accumulator, already scaled and kept in
FP32, so the correction happens before the final BF16 cast. Because the
reconstructed $\widetilde r$ can differ from $r$ and casting
$\widehat\lambda_i^Kq_i$ adds error, $dQ$ and $dK$ approximate the same ideal
projection without sharing a bitwise identical score operand. The key pass
leaves $O$, $dQ$ and $dV$ untouched.

\paragraph{Cost.}
Relative to the native backward, GProj adds one $rK$ contraction, a
query-shaped FP32 workspace, two FP32 row-statistic arrays and the second pass. It saves
no extra forward state and stores no quadratic matrix; $dV$ is not projected.

\section{Gauge Projection: Proofs}
\label{app:gproj-theory}
\label{app:theory-proofs}
All statements concern exact attention at fixed represented operands and
a fixed incoming derivative.

\subsection{Projection and the error that remains}
With $m>0$, $\Pi_r=I-r\mathbf1^\top/m$ obeys
\[
 \Pi_r^2=I-2r\mathbf1^\top/m
       +r(\mathbf1^\top r)\mathbf1^\top/m^2=\Pi_r,
 \qquad \mathbf1^\top\Pi_r=0.
\]
It acts as the identity on vectors of zero row sum. Since $g$ is such a
vector, $\Pi_rt-g=\Pi_r(t-g)$, and contracting the latter gives
\[
 \alpha\sum_j(\Pi_re)_jk_j
 =\alpha\sum_je_jk_j-\alpha\frac{\sum_je_j}{m}\sum_jr_jk_j
 =\alpha\sum_je_j(k_j-\mu_r).
\]
The same algebra gives translation invariance at fixed $t,r$ and completes the
proof of Proposition~\ref{prop:gproj-projection}. The projection removes only
the row mass of the error; an error that already sums to zero is left intact.

\paragraph{Centering before the cast.}
The remaining error depends on how the operand is centered before it is cast.

\begin{lemma}[Conditional centered-cast bound]
\label{lem:gproj-centered-cast}
Suppose $p_j>0$, $\sum_jp_j=1$, and the exact pre-cast operand is
$x_j=p_j(a_j-\beta)$. Suppose its cast $t_j=x_j+\xi_j$ satisfies
$|\xi_j|\le u_b|x_j|$. If all subsequent projection and contraction
arithmetic is exact and the correction uses $r=p$, then
\begin{equation}
 \left\|\alpha\sum_j(\Pi_pt)_jk_j-G_Q\right\|_2
 \le |\alpha|u_b\,\sigma_k
       \sqrt{\sigma_a^2+(\mu_a-\beta)^2},
 \label{eq:gproj-cast-bound}
\end{equation}
where $\sigma_k^2=\sum_jp_j\|k_j-\mu_k\|_2^2$ and
$\sigma_a^2=\sum_jp_j(a_j-\mu_a)^2$.
\end{lemma}

\begin{proof}
The exact projection gives
$\Pi_px=p\odot(a-\mu_a)=g$, irrespective of $\beta$. Thus the query
error is $\alpha\sum_j\xi_j(k_j-\mu_k)$. Weighted Cauchy--Schwarz gives
\[
 \Big\|\sum_j\xi_j(k_j-\mu_k)\Big\|_2
 \le\left(\sum_j\frac{\xi_j^2}{p_j}\right)^{1/2}\sigma_k
 \le u_b\left(\sum_jp_j(a_j-\beta)^2\right)^{1/2}\sigma_k.
\]
The final sum is $\sigma_a^2+(\mu_a-\beta)^2$.
\end{proof}

The error grows with $|\mu_a-\beta|$, so an accurate center before the cast
still matters: projection cannot restore value differences lost in the cast.

\subsection{Proof of Theorem~\ref{thm:fwd-vs-proj}}
With a perfect forward, $x=g$ and $e=t-g=\xi$. Part (a): contracting $t$ gives
$\alpha\sum_jt_jk_j-G_Q=\alpha\sum_j\xi_jk_j$, and adding and subtracting
$\mu_k$ yields the stated decomposition, exactly as in
Equation~\eqref{eq:leakage-decomposition}. For the lower bound, choose the
admissible errors $\xi_j=u_b|g_j|$, all of one sign, so that
$\sum_j\xi_j=u_b\sum_j|g_j|$. By the triangle inequality,
\[
 \|E_{\rm fwd}\|\ge|\alpha|\,u_b\|\mu_k\|\sum_j|g_j|
 -|\alpha|\Big\|\sum_j\xi_j(k_j-\mu_k)\Big\|,
\]
and weighted Cauchy--Schwarz bounds the last norm by
$(\sum_j\xi_j^2/p_j)^{1/2}\sigma_k\le u_b(\sum_jp_j(a_j-\mu_a)^2)^{1/2}\sigma_k
=u_b\sigma_a\sigma_k$, using $|g_j|=p_j|a_j-\mu_a|$. A common key translation by
$b$ leaves $p$, $a$, $g$ and $G_Q$ unchanged and replaces $\mu_k$ with
$\mu_k+b$ while $k_j-\mu_k$ is unchanged; taking $\|b\|\to\infty$ along any
direction makes the bound, and hence the worst-case relative error, unbounded.
Part (b) is Lemma~\ref{lem:gproj-centered-cast} with $\beta=\mu_a$: the projection
with $r=p$ maps $x$ to $g$, the error is $\alpha\sum_j\xi_j(k_j-\mu_k)$, and the
same Cauchy--Schwarz step gives $|\alpha|u_b\sigma_k\sigma_a$. Neither $\mu_k$
nor $b$ enters. \qed

Dynamic shifting
\citep{qiu2026lowprecision} and FP32 saved outputs change only how $p$ and $\mu_a$
are formed, so with a perfect forward they fall under part (a). In the
implemented kernel, $r$ is a separate BF16 cast with $m\ne1$, the center comes
from a saved output and the contractions round; the analysis below accounts for
these terms.

\paragraph{Synthetic offset sweep.}
Panel 2 of Figure~\ref{fig:teaser} checks the theorem numerically. Each of 256 rows
has $L=1024$ keys $k_j=z_j+b$ in $d=64$ dimensions, with $z_j\sim\mathcal N(0,I)$ and a
common translation $b$ along a random direction, $\|b\|\in[1,10^4]$. Queries give logits
with standard deviation 3, and values and the upstream derivative are standard normal.
The forward and the score gradient $g$ are exact in FP64; the only rounding is
$t=\mathrm{BF16}(g)$. The forward-repair curve contracts $t$ directly, GProj subtracts
$(\rho/m)\,r$ with $r=\mathrm{BF16}(p)$, and both contractions are evaluated in FP64.
The translation leaves $p$, $g$ and $G_Q$ unchanged; the horizontal axis is the median
of $\|\mu_k\|/\sigma_k$ at each offset. Logit standard deviations of 1 and 6 change the
forward-repair curve by less than 10\% and move the GProj floor between 0.12\% and 0.17\%.

\subsection{Rounding in the implemented query correction}
The kernel adds accumulation and division errors to the ideal projection. To
separate them, define exact sums of the formed operands by
$A=\sum_jt_jk_j$, $B=\sum_jr_jk_j$, $\rho=\sum_jt_j$ and $m=\sum_jr_j>0$.
Write the actual FP32 accumulator/statistic errors as
\[
 \widehat A=A+e_A,\quad \widehat B=B+e_B,\quad
 \widehat\rho=\rho+e_\rho,\quad \widehat m=m+e_m,
 \quad \widehat\lambda^Q=\lambda+e_\lambda^Q,\quad\lambda=\rho/m.
\]
Let $e_F$ include the rounding of the final multiply/subtract, scale and
BF16 conversion. When all quantities in the decomposition are finite, the
implemented query error satisfies
the exact decomposition
\begin{equation}
 \widehat G_Q-G_Q
 =\alpha\sum_j(t_j-g_j)(k_j-\mu_r)
   +\alpha e_A-\alpha\widehat\lambda^Q e_B
   -\alpha e_\lambda^Q B+e_F.
 \label{eq:gproj-rounding-ledger}
\end{equation}
The first term is the operand error that survives the projection; the others
are the accumulation, ratio and final rounding errors, and the norm of the total
is at most the sum of their norms.

If $|e_m|<m$ and
$\widehat\lambda^Q=\widehat\rho/\widehat m+e_{\rm div}^Q$, then
\begin{equation}
 |e_\lambda^Q|\le
 \frac{|e_\rho|+|\lambda|\,|e_m|}{m-|e_m|}+|e_{\rm div}^Q|.
 \label{eq:gproj-ratio-error}
\end{equation}
This follows by subtracting $\rho/m$ from
$(\rho+e_\rho)/(m+e_m)$. The hypothetical coefficients $t-\widehat\lambda^Q r$ have row
mass $-m e_\lambda^Q$ even before the contraction is rounded.
A row with zero
computed mass receives no correction.

\subsection{Rounding in the implemented key correction}
The key pass adds two further sources of error: it reconstructs the
probabilities and casts the scaled queries. Index every query/head row by $i$, and sum only over rows legally attending
to a key in its KV head. The ideal projected key gradient is
\[
 G_{K,j}^{\rm proj}=\alpha\sum_i(t_{ij}-\lambda_i r_{ij})q_i.
\]
It contracts the same projected operand as Equation~\eqref{eq:gproj-query}; GQA
only enlarges the set of contributing query heads. The implemented second pass instead has a reconstructed BF16 probability
$\widetilde r_{ij}$ and scaled-query operand
$z_i=\operatorname{cast}_{\rm BF16}(\widehat\lambda_i^Kq_i)$.
Put $e_{z,i}=z_i-\widehat\lambda_i^Kq_i$ and
$e_{\lambda,i}^K=\widehat\lambda_i^K-\lambda_i$.
The key ratio is recomputed from the shared $\widehat\rho_i,\widehat m_i$;
write $\widehat\lambda_i^K=\widehat\rho_i/\widehat m_i+e_{\rm div,i}^K$.
This division error need not equal the native query division error.
For $|e_{m,i}|<m_i$,
\[
 |e_{\lambda,i}^K|\le
 \frac{|e_{\rho,i}|+|\lambda_i|\,|e_{m,i}|}{m_i-|e_{m,i}|}
 +|e_{\rm div,i}^K|.
\]
Keeping every operand error, the exact product mismatch is
\begin{align}
 \widetilde r_{ij}z_i-r_{ij}\lambda_iq_i
  &=(\widetilde r_{ij}-r_{ij})\lambda_iq_i
    +\widetilde r_{ij}e_{\lambda,i}^Kq_i
    +\widetilde r_{ij}e_{z,i}.
 \label{eq:gproj-key-mismatch}
\end{align}
If $e_{K,j}$ denotes the error of the already scaled native FP32
accumulator relative to $\alpha\sum_it_{ij}q_i$, $e_{C,j}$ denotes the
FP32 correction-contraction error relative to
$\sum_i\widetilde r_{ij}z_i$, and $e_{F,K,j}$ denotes final arithmetic and
output-cast error, then
\[
 \widehat G_{K,j}-G_{K,j}^{\rm proj}
  =e_{K,j}-\alpha e_{C,j}
   -\alpha\sum_i(\widetilde r_{ij}z_i-r_{ij}\lambda_iq_i)
   +e_{F,K,j}.
\]
Recomputation and the extra BF16 cast are why $dQ$ and $dK$ need not share an
identical effective score operand. Correcting the FP32 accumulator, rather than
the rounded BF16 output, avoids one more rounding step. The value gradient is not
projected.

\section{Comparison Methods and Arithmetic Contracts}
\label{app:comparison-methods}
\begin{table}[t]
\centering\small\setlength{\tabcolsep}{4pt}
\begin{tabular}{lp{3.2cm}p{4.0cm}l}
\toprule
Method & Forward & Backward & Saved output\\
\midrule
FA3 & BF16 FA3 & Native mixed precision & BF16\\
FA3-SBS & Subtract-before-scale FA3 & Native mixed precision & BF16\\
Key smoothing & FA3 on mean-subtracted keys & Native mixed precision & BF16\\
QY-shift & Author-source dynamic shift & Author-source BF16 arithmetic & BF16\\
GProj & Subtract-before-scale FA3 & Dense query/key projection & BF16\\
FP32 attention & FP32 attention & FP32 attention & See below\\
\bottomrule
\end{tabular}
\caption{Primary method definitions. All methods expose BF16 output and
gradient tensors to the model. GProj retains BF16 matrix products with FP32
accumulation and corrections. Eager FP32 recomputes its output; fused FP32
saves an internal FP32 output.}
\label{tab:method-definitions}
\end{table}

\paragraph{Methods.}
Table~\ref{tab:method-definitions} summarizes the methods, which differ only in
the attention computation. \textbf{FA3} is unmodified FlashAttention-3.
\textbf{FA3-SBS} changes only its forward softmax, subtracting the row maximum
before scaling (Appendix~\ref{sec:historical-native-patch}). \textbf{GProj} adds
the query and key corrections to the FA3-SBS forward and, like FA3, saves BF16
$O$ and the FP32 LSE. \textbf{FP32 attention} computes attention in FP32 behind
the model's BF16 interface; we use an eager implementation, which recomputes from
$Q,K,V$, and a fused one, which saves an FP32 output.

\textbf{QY-shift} is the dynamic-shift method of Qiu and Yao
\citep[Section~4]{qiu2026lowprecision}, with the arithmetic of their released
BF16 source (\texttt{STABLE=1}). For a current tile of scaled,
masked logits, let $r$ be its row maximum and
$c=\sum_j\mathbf1\{s_j\geq\operatorname{fl}_{\rm BF16}(r-\epsilon)\}$.
The source computes
\begin{equation}
 b=\begin{cases}
 2r,&r>0\ \text{and}\ c>1,\\
 0,&r<0\ \text{and}\ c>1,\\
 r,&\text{otherwise},
 \end{cases}
 \qquad m_{\rm new}=\max(m_{\rm old},b),
 \label{eq:qy-shift}
\end{equation}
with $\epsilon=10^{-3}$. Preserving the rounded subtraction in the count
predicate matters in BF16. The previous numerator and denominator are
rescaled by $\exp(m_{\rm old}-m_{\rm new})$ before the online update.
The implementation retains $512\times512$ tiles, the authors' $10^{-10}$
tile-sum clamp, BF16 intermediate state and gradient accumulation, and
BF16 saved $O$ and LSE.

The source keeps BF16 precision between operations as well, not only at its
interface. Writing $dP=UV^\top$ and $D$ for the computed row reduction
of $U\odot O$, the source rounds that elementwise product before
its row reduction, forms the scaled score operand as
$(P\alpha)\odot(dP-D)$ in BF16, and updates BF16 gradient buffers after
each tile. GProj instead retains FP32 gradient accumulators and applies
$\alpha$ after the query correction.

Our C++/CUDA port matches the author code byte for byte
(Appendix~\ref{app:gproj-experiments}).

\paragraph{Ablations.}
The ablations build GProj up from FA3-SBS, which shares its forward and keeps the
native backward. \emph{GProj-Q} adds the query correction, together with the new
backward tiling, and \emph{GProj} adds the key pass; GProj always denotes the
complete method. Because all three share one forward, their differences are
differences of the backward, and the key pass is measured before and after
correction on the same execution.

\section{Accuracy and Cost: Setup and Additional Results}
\label{app:gproj-experiments}
\paragraph{Captured inputs.}
We evaluate all kernels on eight attention inputs captured from the
from-scratch FA3 run: layers 5 and 11 at 23.1B and 33.6B tokens (checkpoints
5500 and 8000), on one PG19 book and one ProofPile document
(Table~\ref{tab:gproj-capture-map}). Each capture holds BF16 $Q,K,V$ and the
incoming derivative $U$ for one causal sequence of 4096 tokens with 16 query
heads, four KV heads, head dimension 64 and scale $1/8$. Accuracy can therefore
be reproduced from the captures alone, without checkpoints or the trainer.

\paragraph{Reference and metric.}
The reference is the exact attention VJP at the represented BF16 inputs,
computed in FP64 with centered scores and values and with all GQA
contributions summed. We report $100\|\widehat G-G^{64}\|_2/\|G^{64}\|_2$ for each
full gradient tensor, and medians and maxima over the eight captures. The FP32
attention study evaluates the same target with an independent FP64
implementation (explicit softmax Jacobian, 128-query chunks), and QY-shift is
scored against a separately implemented centered FP64 reference.
Table~\ref{tab:gproj-per-capture} lists the errors of every capture.

\GProjPerCaptureTable

\begin{table}[t]
\centering\small
\caption{Identity key for Table~\ref{tab:gproj-per-capture}. Layers are zero-based;
P is a PG19 book and D a ProofPile document.}
\label{tab:gproj-capture-map}
\begin{tabular}{rrrr}
\toprule
Capture & Checkpoint & Layer & Document\\
\midrule
0 & 5500 & 11 & P \\
1 & 5500 & 5 & P \\
2 & 5500 & 11 & D \\
3 & 5500 & 5 & D \\
4 & 8000 & 11 & P \\
5 & 8000 & 5 & P \\
6 & 8000 & 11 & D \\
7 & 8000 & 5 & D \\
\bottomrule
\end{tabular}
\end{table}

\paragraph{Backward-only comparison.}
GProj uses the FA3-SBS forward, and on every capture its output is
byte-identical to that of FA3-SBS, so comparing FA3-SBS with GProj compares
backward passes only. The GProj backward targets Hopper's SM90a instructions
with head dimension 64 and a query tile of 96, against 128 in FA3. The query
correction alone (GProj-Q) leaves $O$, $dK$ and $dV$ unchanged in 96 checks over
the captures, 20 additional layouts and lengths, and four key-translation cases,
so its median $dK$ error stays at $13.2651\%$. The key pass then corrects the
FP32 key accumulator before the final BF16 cast, using BF16 probabilities and
BF16 scaled queries with FP32 accumulation. On the same execution it lowers the
median $dK$ error from $13.2651\%$ to $0.3712\%$ while leaving $O$, $dQ$ and $dV$
unchanged.

\paragraph{QY-shift.}
QY-shift (Appendix~\ref{app:comparison-methods}) ports the released Qiu--Yao BF16
code without guards or fallbacks, moving the tile loops to C++ with CUDA
pointwise kernels and cuBLAS matrix products. On nine synthetic cases (tile
boundaries, repeated and large logits, noncontiguous tensors, packed segments,
GQA) and on the eight captures, its output, LSE and gradients match the author
code byte for byte. Autocast and compilation are disabled.

\paragraph{Fused FP32 baseline.}
We built a
fused FP32 kernel. It decodes the BF16 inputs, computes in FP32 SIMT/FMA
arithmetic without TF32, and returns BF16 outputs and gradients. It saves FP32
$O$ and three FP32 row statistics: the LSE, the row maximum and the inverse
normalizer, the last two keeping normalization accurate when large logits leave
the LSE short of resolution. At 64 sequences per GPU the saved state is
$1.046875$\,GiB, and no $N^2$ matrix is stored.

\paragraph{Timing.}
Table~\ref{tab:gproj-training} measures complete training steps of the 450M
model on one H200 with batch one and 4096 tokens, under the training
configuration: FP32 master weights, BF16 compute, the Muon/AdamW optimizer,
clipping at one and activation checkpointing. Each step runs 48 attention
forward calls, including recomputation, and 24 backward calls; data loading and
communication are excluded. Each method runs on three workers, each timing 15
updates after warmup, and we report the median over workers of the mean step
time. The two panels are separate jobs on different nodes. Peak allocated memory is $9.191$\,GiB for both FA3 and GProj and
$9.182$\,GiB for eager FP32.

\section{Synthetic Numerical Suite}
\label{app:synthetic-suite}
The synthetic suite contains 58 cases that stress the layouts and numerical regimes an attention kernel must handle, together with two analytic four-key calibrations. The cases
cover packed segments, unequal lengths, empty causal rows, singleton and
tile-boundary rows, batch size two, non-causal attention, common key and value
shifts, constant values, sharp attention, zero and tiny upstream gradients, and
four 4096-token sequences, all in BF16 with head dimension 64 and GQA. The suite takes 65 seconds on
one H200 and evaluates $O,dQ,dK,dV$ for FA3, FA3-SBS and GProj.
Table~\ref{tab:synthetic-suite} summarizes the errors.

\paragraph{References.}
Each case is scored against two independent FP64 references, a centered VJP and
an explicit softmax-Jacobian VJP. They agree to within $10^{-10}+10^{-9}$ times
the target norm on all 60 cases, the worst case using 0.5154\% of that tolerance
($6.30643\times10^{-10}$), and eleven small cases further agree with autograd to
$9.99\times10^{-16}$.

\paragraph{Outcome.}
A case fails if its relative error exceeds 5\%, or if its absolute error
exceeds $10^{-6}$ where the target tensor is zero. GProj passes every case and never
increases the $dK$ or $dV$ error; its query error exceeds that of FA3-SBS in the
three cases below. Its largest relative $dQ$ and $dK$ errors in
Table~\ref{tab:synthetic-suite} both occur with tiny upstream gradients.
\begin{center}\small
\begin{tabular}{lrr}
\toprule
Case & FA3-SBS (\%) & GProj (\%)\\
\midrule
015: unequal lengths & 0.227716 & 0.227779\\
048: tiny upstream & 2.448365 & 2.493797\\
058: analytic calibration & 0.194175 & 0.582524\\
\bottomrule
\end{tabular}
\end{center}

\paragraph{Exact zeros.}
Where the exact answer is zero, GProj stays close to it. The singleton case 000
has exactly zero $dQ$ and a $dK$ residual of $5.94918\times10^{-10}$, and the
constant-value cases 028 and 041 have $dQ/dK$ residuals of
$2.02560\times10^{-9}/2.72331\times10^{-9}$ and
$2.43146\times10^{-9}/3.05945\times10^{-9}$. Under random common-key shifts the
zero-target coordinate reaches at most 0.001953125 (case 053) and 0.01171875
(case 057), against 432 for FA3-SBS in the latter.

\paragraph{Operand masses and structure.}
Over 295,008 query-head rows the BF16 probability mass $m$ is always positive
and finite, with minimum 0.9967041015625 and $|m-1|$ at most 0.00341796875. All
accumulator mappings and second-pass checks pass, GProj's $O$ and $dV$ are
byte-identical to those of FA3-SBS, and repeated runs of one packed and one
4096-token case are bitwise reproducible.

\section{Training Study}
\label{app:training-status}
\paragraph{Setup.}
All matched runs share one recipe: 16 H200 GPUs, seed 42, the same data order,
and 11,930 updates of 4,194,304 tokens (50.0B tokens). They differ only in the
attention computation. Key smoothing subtracts the per-document, per-KV-head
mean key before calling unmodified FA3 (FP32 accumulation, BF16 result);
softmax is invariant to this shift in exact arithmetic, and autograd applies
the matching correction to $dK$. Key smoothing keeps the FA3 forward; FA3-SBS,
which repairs only the forward, still enters the same large-logit regime
(Table~\ref{tab:regime-audit}). The Qiu--Yao
run was stopped at \ScratchQyTokens{} tokens, for reasons given below.
Table~\ref{tab:scratch-progress} and Figure~\ref{fig:qy-collapse} report the
all-rank mean loss and mean pre-clipping gradient norm of each ten-update
logging interval, without smoothing; when a run resumed from a checkpoint, the
resumed records replace the superseded ones.

\begingroup\let\ScratchOrigCaption\caption
\renewcommand{\caption}[1]{\ScratchOrigCaption{Matched from-scratch training (seed 42, same data order):
all-rank mean training loss and mean pre-clipping gradient norm per
10-step logging interval. Cells show loss / gnorm. All five arms complete
50.0B tokens. The Qiu--Yao arm, stopped at \ScratchQyTokens{} tokens, is reported in
Appendix~\ref{app:training-status} (Figure~\ref{fig:qy-collapse}).
``Final 0.84B'' gives mean loss / median gnorm over the 20 records in the
final 0.84B tokens; ``Norm $>$10 at'' gives the tokens at which the gradient
norm first exceeds 10.}}%
\ifdefined\ScratchTableShown\else
\global\let\ScratchTableShown\relax
\begin{table}[t]
\centering
\scriptsize
\setlength{\tabcolsep}{3pt}
\begin{tabular}{rccccc}
\toprule
Tokens (B) & FA3 & FA3-SBS & GProj & Key smoothing & FP32 attention \\
\midrule
0.8 & 4.624 / 0.827 & 4.635 / 0.790 & 4.623 / 0.918 & 4.618 / 0.804 & 4.625 / 0.852 \\
12.6 & 1.870 / 0.406 & 1.871 / 0.360 & 1.869 / 0.312 & 1.870 / 0.373 & 1.869 / 0.372 \\
23.1 & 1.783 / 0.358 & 1.779 / 0.375 & 1.773 / 0.403 & 1.776 / 0.329 & 1.772 / 0.359 \\
29.4 & 1.909 / 3.871 & 1.783 / 0.325 & 1.773 / 0.385 & 1.790 / 0.355 & 1.773 / 0.288 \\
33.6 & 2.122 / 375.508 & 1.780 / 0.364 & 1.769 / 0.396 & 1.937 / 7.381 & 1.769 / 0.354 \\
37.7 & 2.107 / 446.931 & 1.753 / 0.279 & 1.740 / 0.324 & 1.932 / 4.789 & 1.740 / 0.375 \\
41.9 & 1.949 / 70.725 & 1.749 / 0.457 & 1.728 / 0.335 & 2.026 / 138.966 & 1.728 / 0.270 \\
46.1 & 1.882 / 110.166 & 1.679 / 0.210 & 1.664 / 0.142 & 1.912 / 138.389 & 1.664 / 0.168 \\
50.0 & 1.871 / 39.015 & 1.684 / 0.125 & 1.670 / 0.076 & 1.907 / 287.047 & 1.670 / 0.095 \\
\midrule
Final 0.84B & 1.866 / 30.705 & 1.679 / 0.112 & 1.665 / 0.087 & 1.903 / 233.384 & 1.665 / 0.097 \\
Norm $>$10 at & 29.9B & never & never & 33.2B & never \\
\bottomrule
\end{tabular}
\caption{Matched from-scratch training (seed 42, same data order):
all-rank mean training loss and mean pre-clipping gradient norm per
10-step logging interval; logged interval
maxima are not used. Cells show loss / gnorm. All five arms complete
50.0B tokens. The Qiu--Yao arm, stopped at \ScratchQyTokens{} tokens, is reported in
Appendix~\ref{app:training-status} (Figure~\ref{fig:qy-collapse}).
``Final 0.84B'' gives mean loss / median gnorm over the 20 records in the
final 0.84B tokens; ``Norm $>$10 at'' gives the tokens at which the gradient
norm first exceeds 10.}
\label{tab:scratch-progress}
\end{table}
\fi
\endgroup

\paragraph{Large-logit regime.}
Each run's checkpoint at 33.6B tokens (and FA3's at
23.1B) is evaluated with exact FP32 attention on two fixed documents.
Table~\ref{tab:regime-audit} reports the regime statistics and the error that
unmodified FA3 and cuDNN would make on those operands.

\begin{figure}[t]
\centering
\includegraphics[width=\linewidth]{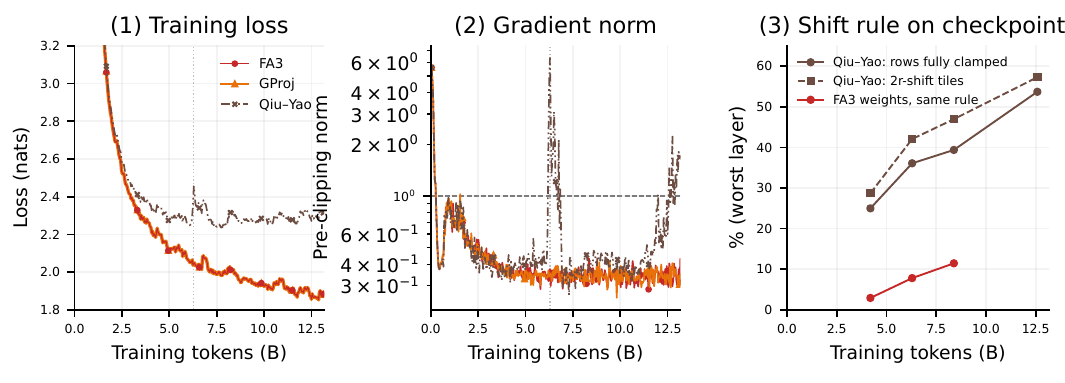}
\caption{The Qiu--Yao baseline arm. (1) All-rank mean training loss and
(2) pre-clipping gradient norm of the Qiu--Yao, FA3 and GProj arms over the
Qiu--Yao arm's \ScratchQyTokens{} tokens, all-rank means without
smoothing; the dotted line marks the spike at 6.3B tokens (step 1500).
(3) Diagnosis on the Qiu--Yao checkpoints (eight fixed
documents, every sixteenth query row): the worst-layer percentage of
query rows whose every 512-key tile sum falls below the source's
$10^{-10}$ clamp, and of legal row--tile pairs taking the $2r$ shift
branch; the FA3 arm's checkpoints under the same rule are shown for
comparison. The text of this appendix gives the mechanism.}
\label{fig:qy-collapse}
\end{figure}

\paragraph{Qiu--Yao run.}
The Qiu--Yao run falls behind FA3 from about 2.5B tokens, spikes at 6.2--6.8B
tokens (gradient norm 6.4) and then stalls near a loss of 2.28--2.30 while FA3
keeps improving, trailing it by 0.31 nats at 8.4--8.9B tokens
(Figure~\ref{fig:qy-collapse}). The cause is the dynamic shift rule, not the
attention backward. When a tile's row maximum $r>0$ has a near-tie ($c>1$), the
rule shifts by $2r$, so every weight is at most $\exp(s-2r)\le\exp(-r)$. Once
logits are large, near-ties are common, because the BF16 spacing at $|s|\ge128$
is 1, and $\exp(-r)$ leaves the FP32 normal range for $r>87$. The tile sum then falls below
the source's $10^{-10}$ clamp, and the row's output is dominated by rounding.

Replaying the author arithmetic on the saved checkpoints (eight documents,
every sixteenth query row) shows this failure growing with training. At step
1000 (4.2B tokens) the $2r$ branch is taken on 12--29\% of row--tile pairs in
layers 4--23, 5--25\% of query rows have every tile clamped, and the attention
output differs from FP32 attention on the same operands by a median relative
error of 0.5--0.9. At step 1500 (6.3B) the branch fraction reaches 18--42\%, up
to 36\% of rows are fully clamped and row maxima reach $r\approx5{,}300$; by step
3000 (12.6B) about half of the rows in layers 15--23 are fully clamped and the
output is uncorrelated with FP32 attention (median relative error $\approx1.0$).
The corrupted output in turn feeds logit growth: the run's 99th-percentile row
maxima are 5--100$\times$ those of FA3 at the same steps, where the same rule
would fully clamp at most 11\% of rows. The weights also adapt to the clamped
outputs: substituting exact FP32 attention does not improve the fixed-document
NLL from step 1000 to 2000 (3.32 to 3.30, and 3.31 at step 3000), whereas the
FA3 arm improves from 3.16 to 2.92 over the same steps; at step 1500 the
substitute is even worse than the run's own arithmetic (3.61 versus 3.51).

\section{Model and Training Configuration}
\label{app:model-configuration}
The 450M model in the main text is a Full-AttnRes transformer
(Table~\ref{tab:model-configuration}),
using learned softmax aggregation over earlier sublayer outputs
\citep{kimi2026attentionresiduals}. Its shared-read configuration uses one
residual-depth read jointly for token-attention queries, keys and values.
All training runs, captured operands and frozen-state diagnostics use
this configuration, and every run is trained from scratch.

\begin{table}[ht]
\centering\small
\begin{tabular}{lp{8.3cm}}
\toprule
Setting & Configuration\\
\midrule
Parameters & 449,465,344; referred to as 450M in the main text\\
Transformer & 24 layers, hidden width 1024\\
Attention & 16 query heads, four key/value heads, head dimension 64\\
Vocabulary & 64,256; untied input and output embeddings\\
Sequence length & 4096 tokens\\
Attention backend & Native BF16 FlashAttention-3\\
Normalization controls & Query--key normalization and softcapping disabled\\
Parameter precision & BF16 compute parameters, FP32 master parameters\\
Matrix optimizer & Muon, learning rate 0.01, momentum 0.99\\
Other parameters & AdamW, learning rate 0.001, $\beta=(0.9,0.95)$\\
Schedule & Warmup 300 updates, constant, cosine decay over the final 9,544 updates to 1\%\\
Regularization & Zero weight decay; no dropout\\
Gradient clipping & Global norm threshold 1\\
Execution & Two nodes, 16 ranks, 64 sequences per rank;
global batch 1024 sequences\\
Activation recomputation & Enabled layerwise\\
Training length & 11,930 optimizer updates (50.0B tokens), seed 42, deterministic\\
\bottomrule
\end{tabular}
\caption{Configuration of the 450M model and of every training run.}
\label{tab:model-configuration}
\end{table}

The comparison arms change only the attention computation. The plain-transformer
control has its own configuration (Appendix~\ref{app:plain-control}).

\section{Localizing the Numerical Failure}
\label{sec:historical-numerics}
All analyses use checkpoints of the
from-scratch FA3 run.

\subsection{The saved-output channel predicts the local error}
\label{sec:historical-saved-output}
FA3's backward forms the reduction $\delta_i=\sum_jp_{ij}u_i^\top v_j$ from the
saved output, a known sensitivity of low-precision attention
\citep{qiu2026lowprecision}. With probabilities and $u_i^\top v_j$ exact, an error
in this reduction produces the query-gradient error
\begin{equation}
 E_{\delta,i}=-\alpha(\widehat\delta_i-\delta_i^\star)
                     \sum_jp_{ij}^\star k_j,
 \label{eq:saved-output-channel}
\end{equation}
where $\delta^\star$ and $p^\star$ are FP64 reference values and
$\widehat\delta_i=\operatorname{dot}_{32}(u_i,\widehat o_i)$ is computed from the
native saved output. We test this prediction on native BF16 $Q,K,V,O$ and
incoming derivatives $U$ captured at layers 5 and 11, after checking that
rerunning the native forward reproduces the saved output bitwise, and compare it
with the observed error $E_i=\widehat{dq}_i-dq_i^\star$ through their cosine and
the residual fraction $\|E-E_\delta\|_2/\|E\|_2$. The analysis uses 128 strided
query rows against all 4096 keys, on an exploratory set of eight documents and a
confirmation set of 127 documents at checkpoints 5500 and 8000 (23.1B and 33.6B
tokens); the kernel comparisons elsewhere use all 4096 rows.

One document shows the effect. For a PG19 book at checkpoint 8000, layer 11,
the native $\|dQ\|_2$ is $4.30\times10^{-3}$ against a reference norm of
$4.27\times10^{-5}$, a relative error of 101, whereas the same VJP computed in FP32
has relative error 0.015. The prediction has cosine 0.9999 with the native error
and leaves a residual fraction of 0.020. The forward output, by contrast, has relative
error only 0.00193.

\begin{table}[t]
\centering\small\setlength{\tabcolsep}{4pt}
\begin{tabular}{lrrrrrrr}
\toprule
& & & & & & \multicolumn{2}{c}{After saved $O$}\\
\cmidrule(lr){7-8}
Ckpt. & Layer & $dQ$ rel. err. & Resid. frac. & Err. cos. & Saved-$O$ red. & rel. err. & cos.\\
\midrule
5500 & 5  & 6.517   & 0.0055 & 1.0000 & 2.98$\times$  & 2.165 & 0.419\\
5500 & 11 & 6.120   & 0.0060 & 1.0000 & 2.45$\times$  & 2.496 & 0.366\\
8000 & 5  & 17.889  & 0.0067 & 1.0000 & 15.73$\times$ & 1.055 & 0.686\\
8000 & 11 & 411.104 & 0.0197 & 0.9999 & 72.02$\times$ & 5.988 & 0.156\\
\bottomrule
\end{tabular}
\caption{Confirmation on 127 documents per checkpoint/layer, with 128
strided query rows and all 4096 keys. Entries are per-document medians;
relative errors are ratios, not percentages. Resid.\ frac.\ and Err.\ cos.\
compare Equation~\eqref{eq:saved-output-channel} with observed native error.
Saved-$O$ red.\ is the median paired reduction in absolute $dQ$ error after
replacing only saved $O$ by an FP64-recomputed output rounded to BF16; the last
two columns give the remaining relative $dQ$ error and the cosine between the
computed and exact $dQ$.}
\label{tab:historical-confirmation}
\end{table}

Table~\ref{tab:historical-confirmation} shows that the same holds across 127
documents. At checkpoint 8000, layer 11, the median residual fraction is 0.0197,
and replacing only the saved $O$ reduces the absolute $dQ$ error by a median
factor of 72.02; the earlier checkpoint shows the same channel with smaller
amplification.

\subsection{Fixed-forward interventions reach the full model}
\label{sec:historical-fixed-forward}
We rerun the full backward on eight documents at checkpoints 6500 and
8000 (27.3B and 33.6B tokens) and change only the attention backward of layers 5
and 11: either the native backward receives a correct saved output (recomputed
in FP64 and rounded to BF16), or the exact attention VJP is computed in FP32. All
24 attention outputs and the loss stay bitwise identical, but a change at layer
11 propagates to the derivative entering layer 5, so these are whole-model
effects.

\begin{table}[t]
\centering\small
\begin{tabular}{lrr}
\toprule
Backward intervention at layers 5 and 11 & 6500 & 8000\\
\midrule
Native                                  & 36.653 & 5353.856\\
Replace saved $O$, first row only         & 36.402 & 5353.857\\
Exact first-row VJP only                  & 36.402 & 5353.857\\
Replace saved $O$, all rows               & 2.450  & 19.746\\
FP32 VJP, all rows                        & 2.258  & 18.185\\
\bottomrule
\end{tabular}
\caption{Median full-model parameter-gradient norm over eight documents.
All 24 attention outputs and the scalar loss are bitwise identical across
modes at each document/checkpoint. ``All rows'' still changes only two
attention layers. Saved-$O$ replacement uses FP64 recomputation followed
by BF16 rounding, as in
Table~\ref{tab:historical-confirmation}.}
\label{tab:historical-fixed-forward}
\end{table}

Both repairs remove almost all of the excess gradient
(Table~\ref{tab:historical-fixed-forward}). At checkpoint 8000, correcting the saved
output lowers the median norm from 5353.856 to 19.746, and the FP32 VJP gives
18.185. At checkpoint 6500, before the training gradient norm first exceeds 10,
the native norm is already 36.653, and the same repairs bring it to about
2.3--2.5.

Repairing only the first query row, by contrast, changes almost nothing. That
row has a single legal key, so $p_{00}=1$, $o_0=v_0$ and $dq_0=0$ exactly, and
native captures violate these identities (Section~\ref{sec:historical-native-patch});
yet correcting its saved output or supplying its exact VJP barely changes the
model gradient. A violated exact identity can reveal a kernel defect without
locating the rows that dominate training.

\subsection{The source of the saved-output error}
\label{sec:historical-native-patch}
The first-row violations point to the forward softmax, which forms the
exponential argument
\begin{equation}
 r=\operatorname{fma}\!\left(x,a,-\operatorname{round}_{32}(ma)\right),
 \qquad z=2^r,
 \label{eq:historical-fma}
\end{equation}
where $m$ is the unscaled row maximum and $a$ includes the attention scale and
the base-two conversion. Because the fused multiply-add does not round $xa$
separately, $r$ need not vanish even when $x=m$, a hazard known from prior work
\citep{pytorch2024fma}. If the value product then uses a BF16 cast of $z$ while
the denominator keeps its FP32 value, a one-key row outputs
\begin{equation}
 \widehat o_0\simeq\operatorname{round}_{\mathrm{BF16}}\!\left(
     \frac{\operatorname{round}_{\mathrm{BF16}}(z)}{z}\,v_0\right)
 \label{eq:historical-one-key}
\end{equation}
rather than $v_0$, and this error enters a subtraction whose exact result is zero.

Prescribed operands confirm this model. With all 64 query coordinates equal to
$c2^k$ and all key coordinates equal to one, for $c\in\{3,5,7\}$, $k=0,\ldots,15$
and sequence lengths 1 and 128, the dot products are known exactly; these 96
configurations share 48 constructions. The scalar model reproduces every native
first-row output bitwise, including the six configurations with $o_0\ne v_0$
(Figure~\ref{fig:historical-fma}).

\begin{figure}[t]
\centering
\includegraphics[width=.95\linewidth]{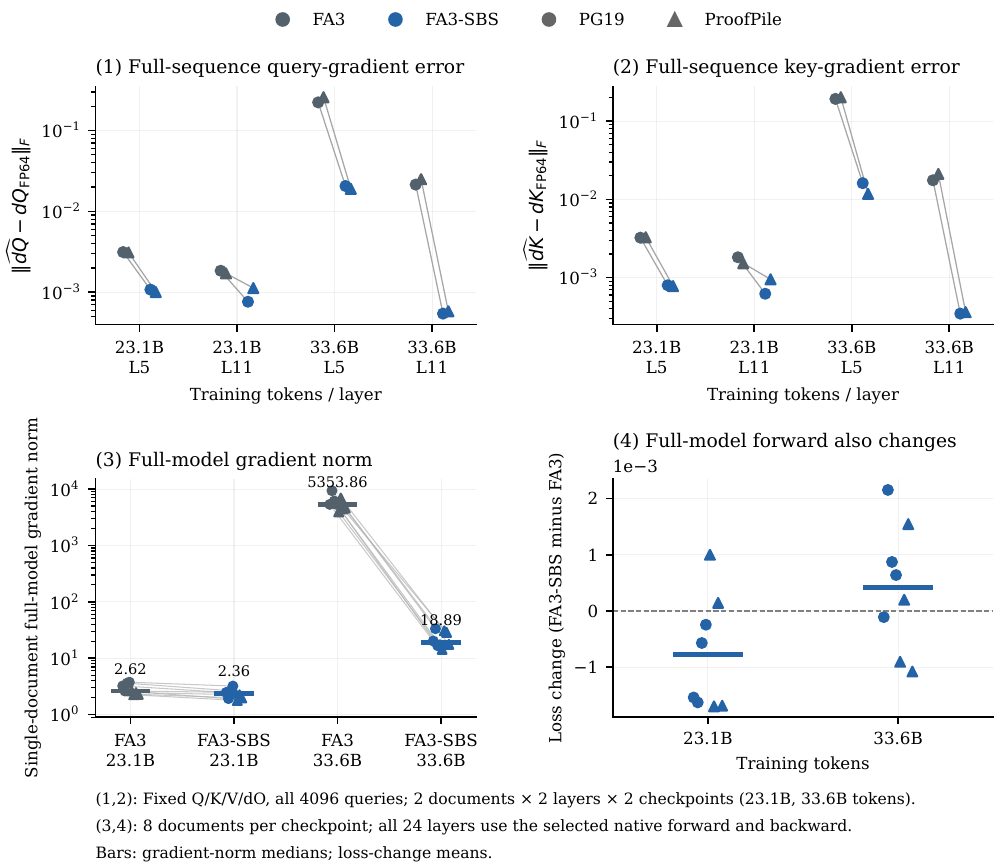}
\caption{Native source intervention on the from-scratch FA3 run. (1)--(2) Eight fixed
captures and incoming derivatives, evaluating all 4096 query rows against
FP64 attention. (3)--(4) Whole-model runs on eight documents per checkpoint,
with the FA3-SBS forward used in all 24 attention layers; unlike the
fixed-forward controls in Table~\ref{tab:historical-fixed-forward}, this
changes the forward computation.}
\label{fig:historical-native-patch}
\end{figure}

Subtracting the row maximum with an explicit rounded operation before the
base-two scaling, with the rest of the source and the compiler settings
unchanged, gives FA3-SBS, the forward repair used throughout; it supplies the
forward of GProj but no backward projection. It restores $o_0=v_0$ and $dQ=0$ in
all 96 prescribed cases and, on the four checkpoint-8000 captures, reduces the
absolute $dQ$ error by $10.8$--$43.0\times$ and the $dK$ error by
$12.0$--$58.6\times$ (Figure~\ref{fig:historical-native-patch}). Using this forward in
all 24 layers lowers the median model-gradient norm at checkpoint 8000 from
5353.856 to 18.892, with mean losses of 3.072690 and 3.073105.

\subsection{What the forward repair leaves behind}
\label{sec:historical-state-cross}
Crossing saved states shows that the improvement comes from the saved output
alone. For each capture we pair the FA3 and FA3-SBS saved outputs with their
LSEs and pass all 32 states to three backward binaries: identical states give
bitwise identical $dQ,dK,dV$ in all three, and on the late captures changing
only $O$ reduces the query-gradient error by $10.8$--$43.0\times$, while changing
only the LSE slightly worsens every case.

The repaired backward is still sensitive to where the keys sit. Setting one
query coordinate to zero and the matching coordinate of every key to a common
BF16 value $b$ leaves scores, output and LSE bitwise unchanged, and the exact
gradient in that coordinate is zero. The native error in that coordinate still
grows linearly with $b$, with FA3-SBS and even with an exact saved output; at
$b=65536$, FA3 and FA3-SBS give 665.15 and 665.54
(Appendix~\ref{app:historical-state-gauge}, Figure~\ref{fig:historical-state-gauge}).
This is the translation-invariance failure that Section~\ref{sec:channel-two} and
the $-24$ witness of Appendix~\ref{app:leakage-details} trace to the BF16 cast,
and that no forward repair can remove.

The failure is not specific to FA3. At fixed captured operands FA2 reproduces
the FA3 errors, while cuDNN and PyTorch's fused SDPA backends behave like
FA3-SBS and retain a substantial backward error
(Appendix~\ref{app:historical-backends}).

\section{Numerical Diagnosis: Details and Additional Controls}
\label{app:historical-numerics}

\subsection{Measurement sets and references}
\label{app:historical-populations}
\paragraph{Measurement sets.}
The local analysis of Appendix~\ref{sec:historical-saved-output} samples query rows
$0,32,\ldots,4064$ against all 4096 keys, on eight exploratory documents and on
127 confirmation documents, at layers 5 and 11 of checkpoints 5500 and 8000
(23.1B and 33.6B tokens). The fixed-forward interventions of
Appendix~\ref{sec:historical-fixed-forward} use all rows of eight documents at
checkpoints 6500 and 8000. The kernel comparisons use the eight complete captures
of Appendix~\ref{app:gproj-experiments}, with all 4096 query rows.

\paragraph{References.}
All references differentiate exact attention in FP64 at the represented BF16
operands. Saved outputs are recomputed in FP64 and rounded to BF16, and saved
LSEs are rounded to FP32, matching the native formats.
Figures~\ref{fig:historical-confirmation-decomposition}
and~\ref{fig:historical-confirmation-interventions} show the full distributions.

\begin{figure}[t]
\centering
\includegraphics[width=\linewidth]{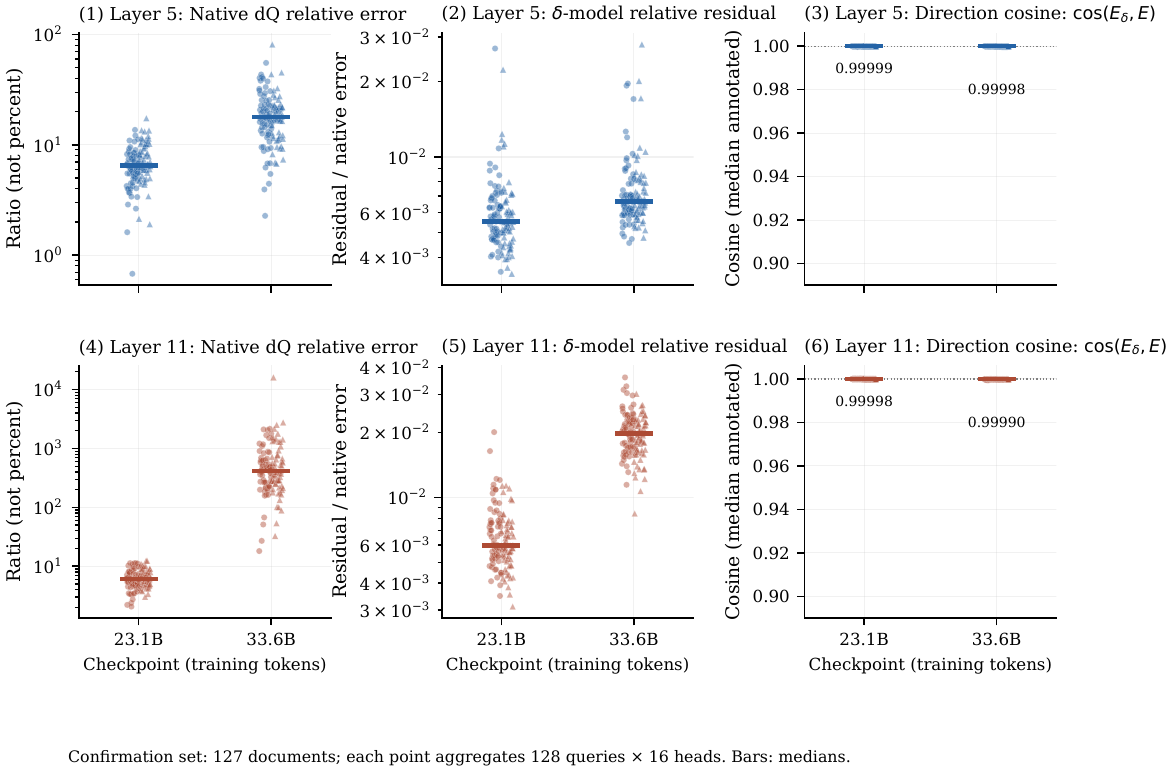}
\caption{The complete 127-document confirmation distributions. Each point
pools 128 sampled query rows and 16 query heads for one document, against
all 4096 keys; bars are medians. Columns show relative native $dQ$ error,
the unexplained fraction after the saved-output prediction, and the error
cosine.}
\label{fig:historical-confirmation-decomposition}
\end{figure}

\begin{figure}[t]
\centering
\includegraphics[width=\linewidth]{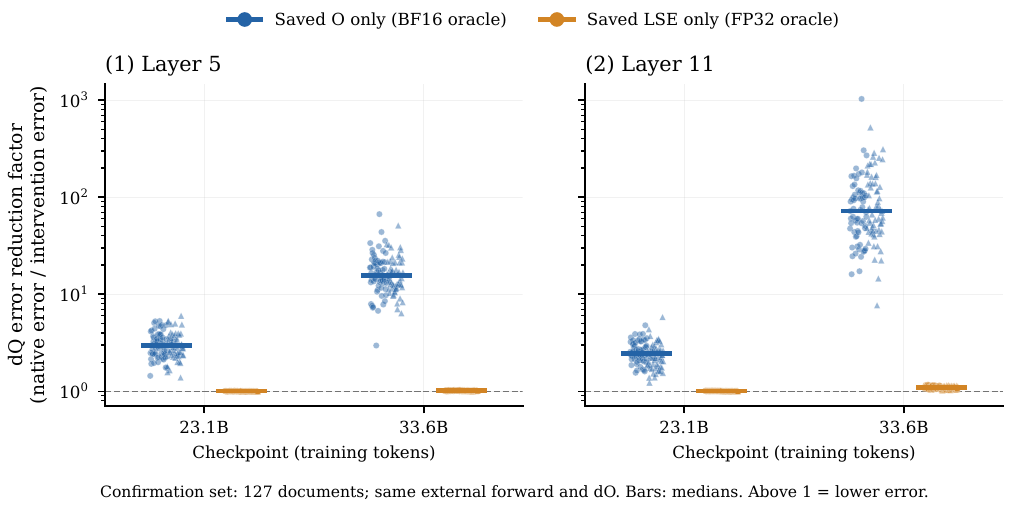}
\caption{Saved-state interventions on the same 127-document confirmation
set and sampled query rows. Each point is a paired absolute-$dQ$-error
reduction, with one as the no-improvement reference. Saved $O$ is recomputed
in FP64 and rounded to BF16; LSE is recomputed in FP64 and rounded to FP32.
Native forward outputs
and incoming derivatives remain fixed.}
\label{fig:historical-confirmation-interventions}
\end{figure}

\subsection{Prescribed FMA cases}
\label{app:historical-source-controls}
The 48 constructions of Appendix~\ref{sec:historical-native-patch} set all 64
coordinates equal, so every dot product is known exactly, and run at lengths 1
and 128 for 96 configurations. The incoming derivative is nonzero only at the
first query, whose support is a single key at both lengths. The scalar model
uses a libm fused multiply-add, an FP32 base-two exponential, BF16 rounding of
the value-product operand and an FP32 denominator; it matches all 96 native
first-row outputs bitwise, including the six violations.

\begin{figure}[t]
\centering
\includegraphics[width=.94\linewidth]{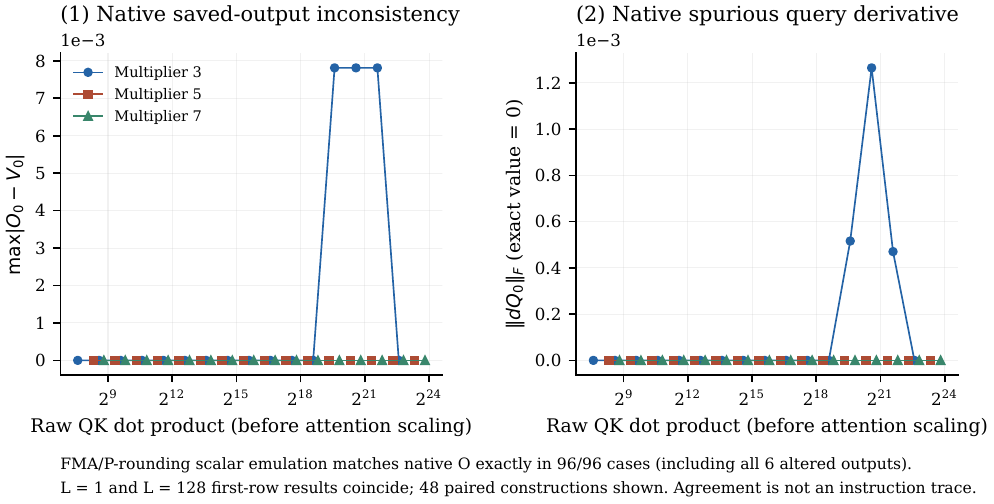}
\caption{Prescribed one-key analysis. The 96 executions share 48
query/key parameter constructions across two sequence lengths. Exact
mathematics requires the output/value identity and a zero query gradient.
The scalar arithmetic hypothesis reproduces the observed output in every
case, including the six violations.}
\label{fig:historical-fma}
\end{figure}

The explicit rounded subtraction of FA3-SBS restores both identities in all 96
prescribed cases, and on the eight real captures its first-row outputs also
equal their values.

\subsection{Saved-state crosses and the remaining gauge defect}
\label{app:historical-state-gauge}
Table~\ref{tab:historical-cross-state} separates the two saved quantities on the
four late captures (Appendix~\ref{sec:historical-state-cross}): the FA3-SBS saved
output alone reproduces the whole improvement, while its LSE alone slightly
worsens every case, as factors below one indicate. Every identical supplied
state gives bitwise identical gradients across the three backward binaries.

\begin{table}[ht]
\centering\small
\begin{tabular}{llrr}
\toprule
Saved $O$ & Saved LSE & $dQ$ error reduction & $dK$ error reduction\\
\midrule
FA3-SBS & FA3 & $10.822$--$42.971$ & $11.968$--$58.601$\\
FA3 & FA3-SBS & $0.99676$--$0.99806$ & $0.99636$--$0.99780$\\
FA3-SBS & FA3-SBS & $10.815$--$43.005$ & $11.954$--$58.562$\\
\bottomrule
\end{tabular}
\caption{Ranges of paired absolute-error reduction on the four checkpoint-8000
full-row captures. The improvement from the source intervention
is recovered by its saved output with the FA3 LSE.}
\label{tab:historical-cross-state}
\end{table}

The gauge construction uses sequence length 128, four query heads, two KV
heads, dimension 64 and scale $1/8$. Query coordinate zero is identically zero,
the matching coordinate of every key is set to $b$, another coordinate carries a
score gap of 0, 1 or 64, and values and upstream derivatives are one fixed
random draw. Across three binaries, all 144 byte-level comparisons confirm that
output and LSE are unchanged as $b$ varies, and in all 126 nonzero-offset
comparisons the erroneous coordinate divided by $b$ is exactly the same vector:
the error is linear in $b$. At gap 64 the slopes are 0.0101494 for FA3 and
0.0101553 for FA3-SBS, which give the reported 665.15 and 665.54 at $b=65536$
(Figure~\ref{fig:historical-state-gauge}). An exact saved output still leaves a
linear error, consistent with the $-24$ witness of
Equation~\eqref{eq:minus24-witness} (Appendix~\ref{app:leakage-details}), which
isolates the cast channel with exact saved output and reduction.

\begin{figure}[t]
\centering
\includegraphics[width=\linewidth]{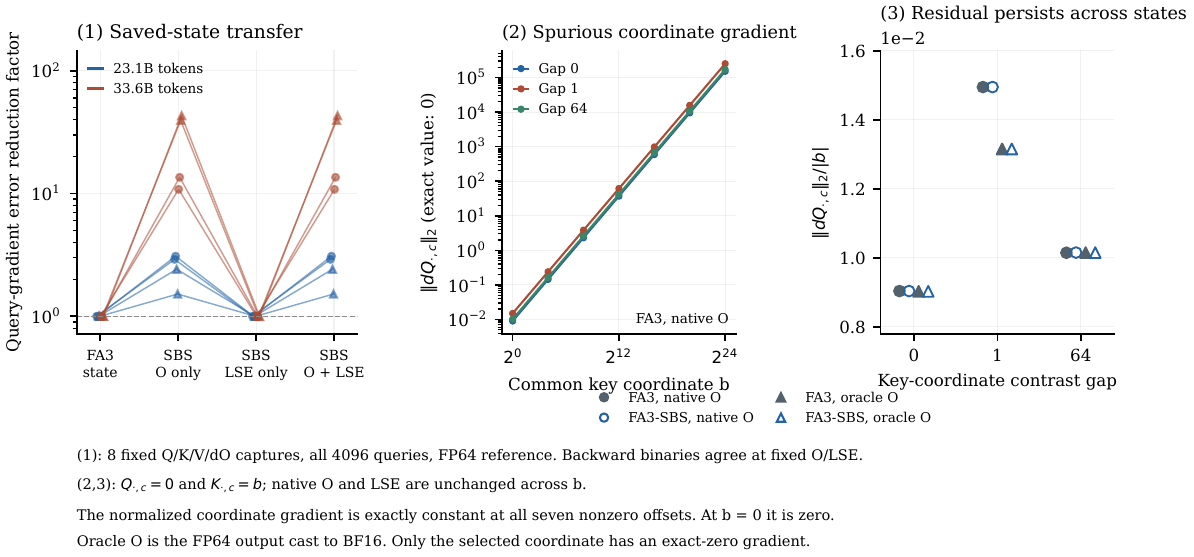}
\caption{Two controls. (1) Saved-output/LSE crosses
on fixed real captures isolate the observed improvement to saved output.
(2)--(3) A synthetic common-key coordinate changes no score or native forward
state but amplifies a mathematically zero query-gradient coordinate.}
\label{fig:historical-state-gauge}
\end{figure}

\subsection{Backend comparison at identical captured operands}
\label{app:historical-backends}
Table~\ref{tab:historical-backends} compares attention backends on the eight
captures at fixed $Q,K,V,U$, against the FP64 reference, using PyTorch 2.8.0,
CUDA 12.8 and cuDNN 9.10 on one H200; the efficient SDPA path needs explicit
KV-head expansion, and cuDNN cannot run the length-one cases.

\begin{table}[ht]
\centering\small\setlength{\tabcolsep}{3pt}
\begin{tabular}{lrrrrc}
\toprule
Backend & $dQ$ error & $dK$ error & $dQ$ cosine & Late L11 $dQ$ error & One-key violations\\
\midrule
FA3                             & 7.729 & 0.8555 & 0.131 & 162.3 & 6/96\\
FA3-SBS                         & 2.187 & 0.1327 & 0.414 & 4.12  & 0/96\\
FA2, 2.8.3.post1                & 7.729 & 0.8555 & 0.131 & 162.3 & 6/96\\
SDPA cuDNN                     & 2.206 & 0.1287 & 0.412 & 4.16  & 0/48\\
SDPA flash                     & 2.217 & 0.1289 & 0.410 & 3.90  & 0/96\\
SDPA efficient, KV expanded    & 2.147 & 0.1271 & 0.422 & 4.04  & 0/96\\
SDPA math, BF16                & 0.0037 & 0.0035 & 1.000 & 0.0077 & 0/96\\
SDPA math, FP32                & 0.0031 & 0.0032 & 1.000 & 0.0074 & 0/96\\
\bottomrule
\end{tabular}
\caption{Fixed-operand backend sweep on the from-scratch FA3 run's captures. Errors are relative
$L_2$ ratios, not percentages; the first three numeric columns are
medians over eight full-row captures. ``Late L11'' is the checkpoint-8000
PG19 layer-11 capture. One-key violations count prescribed cases with
$o_0\ne v_0$ (the same cases have nonzero $dq_0$); cuDNN cannot run the 48
length-one cases.}
\label{tab:historical-backends}
\end{table}

FA2 reproduces FA3: per-capture relative errors differ by less than
$3\times10^{-8}$, and it shows the same six one-key violations. cuDNN and
PyTorch's flash and efficient SDPA backends behave like FA3-SBS in the forward,
yet their median relative query-gradient error remains about 2.2, against
0.003--0.004 for the materialized math paths, and every fused backend returns a
nonzero first-row $dQ$ on all eight captures, including those with $o_0=v_0$. The
backward error is thus shared across fused BF16 backends.

\section{A Plain-Transformer Control}
\label{app:plain-control}
We train three
independently initialized standard-residual transformers (102M parameters; 12
layers, width 512, eight query heads, two KV heads, head dimension 64, no QK
normalization) for 8,000 updates at batch 8 and sequence length 4096, each at a
base learning rate (LR1) and at a fourfold stress (LR4). Periodically, and at
training-batch gradient peaks, we replace every attention backward by an FP32
reference while keeping all attention outputs and the loss bitwise fixed.

In this model the attention backward stays close to the FP32 reference. At 500 updates the full-model
gradient discrepancy is 0.26--0.27\% at LR1 and 1.5--2.4\% at LR4
(Table~\ref{hist:tab:plain-gradient-audits}), and the worst values over all
periodic checks are 0.88\% and 12.4\%. The LR4 runs do become unstable, with
every final update clipped, training NLL of 5.34--5.40 against 3.54--3.57 at LR1,
and gradient-norm peaks of 2,152--4,914. At the peak of seed 2026 (update 5054),
the FP32 backward gives a norm of 4,988.69 against 4,913.88 natively
(Table~\ref{hist:tab:plain-peaks}). The difference from the 450M FA3 run lies in scale: the plain models'
maximum query RMS stays at 42.5--44.5, against above 200 in layer 11 of the FA3
run at 33.6B tokens (Figure~\ref{hist:fig:plain-long}).

\begin{table}[!htb]
\centering\small\setlength{\tabcolsep}{5pt}
\begin{tabular}{rrrrrrr}
\toprule
Seed & LR & Rel. err. 0 & Rel. err. 250 & Rel. err. 500 & Abs. err. 500 & Cosine 500 \\
\midrule
2026 & 1 & 0.00177832 & 0.00187079 & 0.00261074 & 0.00498338 & 0.999996597 \\
2026 & 4 & 0.00177832 & 0.00340782 & 0.0243191 & 0.0514253 & 0.999706727 \\
2027 & 1 & 0.00193911 & 0.00194124 & 0.00264846 & 0.00479022 & 0.999996495 \\
2027 & 4 & 0.00193911 & 0.00345874 & 0.0156458 & 0.0314603 & 0.999878256 \\
2028 & 1 & 0.00186521 & 0.00200166 & 0.0026835 & 0.00528956 & 0.999996407 \\
2028 & 4 & 0.00186521 & 0.00360284 & 0.015028 & 0.0348366 & 0.999897825 \\
\bottomrule
\end{tabular}
\caption{Full-model gradient analyses of three independent 101,986,816-parameter plain transformer initializations. All attention forward outputs and loss are bitwise matched; only all-attention backward is replaced by FP32 reference arithmetic. Relative error uses the reference gradient norm.}
\label{hist:tab:plain-gradient-audits}
\end{table}

\begin{table}[!htb]
\centering\small\setlength{\tabcolsep}{4pt}
\begin{tabular}{rrrrrr}
\toprule
Seed & Peak step & Native norm & Reference norm & Rel. error & Cosine\\
\midrule
2026&5054&4913.875&4988.690&0.018294&0.999944\\
\bottomrule
\end{tabular}
\caption{The LR4 training-batch peak of seed 2026. The analysis replaces every attention backward on the identical native forward and accumulates all eight sequence gradients in FP32 with training's scored-token weights.}
\label{hist:tab:plain-peaks}
\end{table}

\begin{figure}[!htb]
\centering \includegraphics[width=\linewidth]{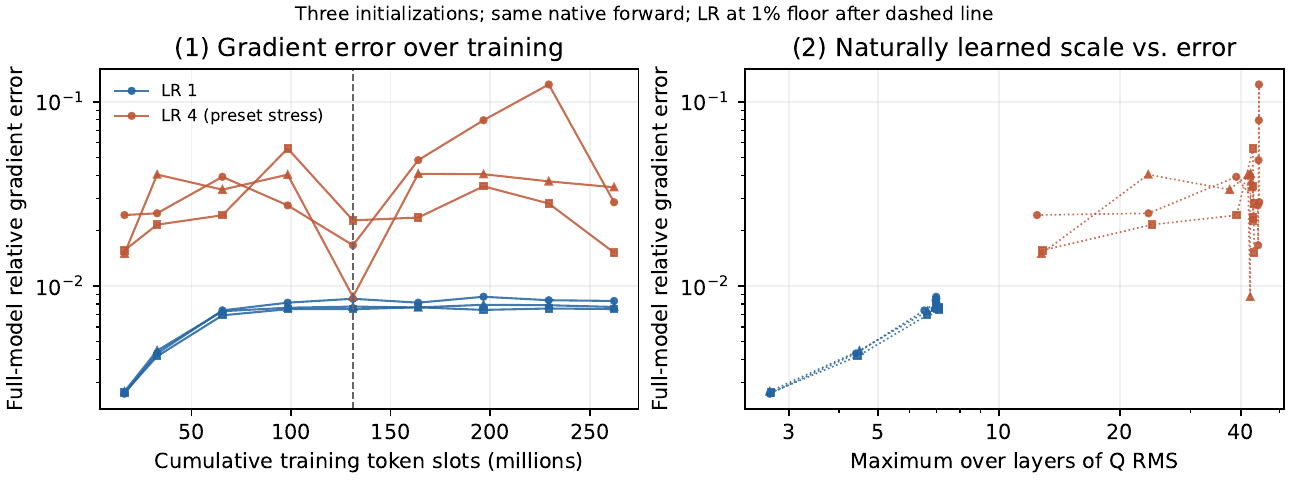}
\caption{All six completed plain-transformer trajectories. The dashed line marks the learning-rate-floor transition at update 4000.}
\label{hist:fig:plain-long}
\end{figure}

\end{document}